\documentclass{article} % For LaTeX2e
\usepackage{times}
\usepackage{iclr2027_conference}

\usepackage{amsmath,amsfonts,bm}

\def\eqref#1{equation~\ref{#1}}
\def\1{\bm{1}}

\DeclareMathAlphabet{\mathsfit}{\encodingdefault}{\sfdefault}{m}{sl}
\SetMathAlphabet{\mathsfit}{bold}{\encodingdefault}{\sfdefault}{bx}{n}

\usepackage{hyperref}
\usepackage{url}
\usepackage{booktabs}
\usepackage{multicol}
\usepackage{multirow}

\usepackage{amsmath}
\usepackage{amssymb}
\usepackage{amsthm}
\usepackage{longtable,pdflscape}

\usepackage[capitalize,noabbrev]{cleveref}

\makeatletter
\@ifundefined{proposition}{
    \newtheorem{proposition}{Proposition}[section]
}{}
\makeatother

\title{\methodname{}: Making Privileged Geometry Supervision Effective for Visuomotor Learning}

\iclrfinalcopy

\author{
Han Fang \\
Global College, Shanghai Jiao Tong University \\
Shanghai, China
\And
Yunpeng Jiang \\
Global College, Shanghai Jiao Tong University \\
Shanghai, China
\And
Jianshu Hu \\
Shanghai Institute of Physical Intelligence and Robotics (PAIR) \\
Shanghai, China
\And
Zhiyuan Guan \\
Global College, Shanghai Jiao Tong University \\
Shanghai, China
\And
Ruiguo Sun \\
East China University of Science and Technology \\
Shanghai, China
\And
Shujia Li \\
Global College, Shanghai Jiao Tong University \\
Shanghai, China
\And
Paul Weng \\
Duke Kunshan University \\
Jiangsu, China
\And
Xiao Li \\
School of Mechanical Engineering, Shanghai Jiao Tong University \\
Shanghai, China
\And
Yutong Ban\thanks{Corresponding author: \texttt{yban@sjtu.edu.cn}} \\
Global College, Shanghai Jiao Tong University \\
Shanghai, China
}

\newcommand{\methodfullname}{\textbf{L}atent \textbf{E}ncoding with \textbf{A}ligned \textbf{P}rivileged Geometry}
\newcommand{\methodname}{LEAP}
\newcommand{\methodintro}{\methodfullname{}  (\methodname{})}

\usepackage{etoolbox}
\newtoggle{final}
\toggletrue{final} % uncomment before submission

\iftoggle{final}{
	\usepackage[disable]{todonotes}
}{
    \setlength{\marginparwidth}{1.5cm}
	\usepackage[size=tiny]{todonotes}
}

\definecolor{darkgreen}{rgb}{0 0.6 0}
\newcommand{\pw}[1]{\iftoggle{final}{#1}{{\color{orange} #1}}}

\begin{document}

\maketitle
\fancyhead{}

\begin{abstract}
Privileged 3D supervision uses additional geometric information during training to guide RGB-based visuomotor policy learning, without requiring geometric inputs at deployment.
However, low reconstruction error does not ensure that visual representations capture geometry useful for control. We identify three limitations that weaken this supervision: proprioceptive shortcuts, dominant-view reliance, and reconstruction objectives dominated by task-irrelevant geometry. To address these limitations, we propose Latent Encoding with Aligned Privileged Geometry (\methodname{}). 
Our framework uses an auxiliary decoder to reconstruct point clouds within the manipulation workspace from visual features alone, while retaining proprioception for action prediction.
Alongside full reconstruction, we introduce wrist-view dropout and partial reconstruction targets matched to the retained wrist, encouraging the encoder to capture complementary local geometry.
The auxiliary decoder is removed at inference, leaving only RGB observations and robot state as policy inputs. 
Extensive experiments on RoboTwin, ManiSkill, and real-world tasks demonstrate consistent and substantial improvements over Diffusion Policy and ACT, with only a small increase in parameter count and no reduction in inference speed.

\end{abstract}

\section{Introduction}
\label{sec:introduction}

Reliable robotic manipulation depends on the 3D relationships among objects and the robot. RGB-based imitation learning methods, such as Diffusion Policy~\citep{chi2023diffusion} and ACT~\citep{zhao2023act}, learn complex skills from demonstrations, but their training objectives do not explicitly require visual features to encode these relationships. 
% \TODO{"action objectives" sounds strange to me.
% Also, it's ambiguous since it could also refer to the goals of the actions.
% I guess you are referring to "training objective" or "training loss".
% Why do you need to specify "action"?}
Point-cloud policies such as DP3~\citep{ze2024dp3} provide explicit geometric inputs, but require depth sensing and point-cloud processing at deployment, adding sensing requirements and computational cost. 
Privileged geometry supervision instead uses 3D information during training, through auxiliary geometric prediction~\citep{li2025qdepth,he2026foresight,kim2026pri4r} or representation alignment~\citep{li2025spatialforcing}, while retaining an RGB policy at deployment. 
% Related work also transfers geometric knowledge through a pretrained visual encoder~\citep{vuong2025evggt}. 
% \TODO{Do we really need this last sentence in the intro?
% That work could instead be mentioned in the Related Work for instance.}

% Recent methods supervise policies through auxiliary geometric prediction~\citep{li2025qdepth,he2026foresight,kim2026pri4r} or alignment with pretrained geometric features~\citep{li2025spatialforcing,vuong2025evggt}.
%\TODO{I guess 3D point-cloud policies have better accuracies than RGB-based policies trained via Privileged geometry.solved}
% I will state clearly this fact and identify the explanation for this gap, which is your next sentence.}
% However, a low geometric loss alone does not establish what information the visual encoder has learned.
% We identify three limitations that can weaken this supervision.
%\TODO{I would list the names of the three limitations here.solved}
\pw{While RGB-based policies are easier to deploy, this line of work suffers from the following potential issue.}
Without explicit 3D observations at deployment, an RGB policy must infer task-relevant geometry from its images.
A low geometric training loss does not establish that it has done so. 
\pw{In this paper, considering policies that take multi-view RGB images and robot state as input, we investigate the following two questions:} 
% \TODO{Does previous work in this direction all use "policies that take multi-view RGB images and robot state as input"? If that's the case, this should be mentioned earlier (e.g., previous paragraph).}
% (1) when does \pw{privileged supervision with depth-derived geometric information} make the visual representation \pw{really} capture task-relevant 3D structure?
(1) when does \pw{privileged supervision with depth-derived geometric information} fail to make the visual representation capture task-relevant 3D structure?
% \TODO{Given the flow of the presentation, it seems it would be better to formulate the first question in the negative: 
% when does \pw{privileged supervision with depth-derived geometric information} fail to make the visual representation capture task-relevant 3D structure?
% Then the second question would be:
% how can privileged 3D supervision be formulated to address/reduce/circumvent this failure?
% }
% (2) how can privileged 3D supervision be formulated to ensure it?
(2) how can privileged 3D supervision be formulated to circumvent this failure?
\pw{To answer the first question, w}e identify three limitations of privileged geometry supervision: \emph{proprioceptive shortcuts}, \emph{view dominance}, and \emph{task-irrelevant global prediction}.
\pw{To address the second, we propose a novel training framework.
We expand on those points next.}

\begin{figure}[t]
    \centering
    \includegraphics[width=0.99\linewidth]{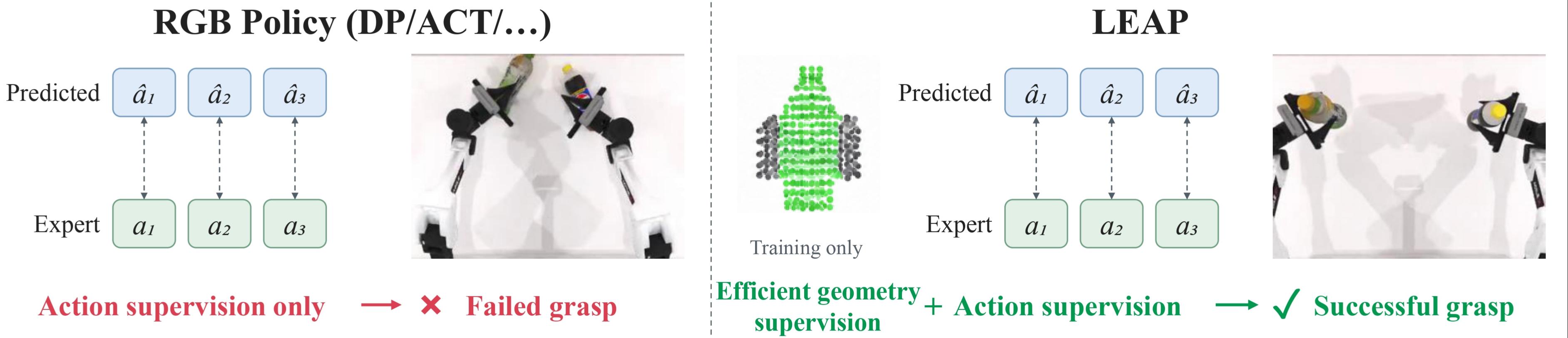}
    \caption{Illustration of \methodname{}’s design goal. \methodname{} learns geometry-aware policies through training-only 3D supervision, without requiring 3D inputs at deployment.}
    \label{fig:teaser}
    \vspace{-2ex}
\end{figure}

A \emph{proprioceptive shortcut} occurs when \pw{an auxiliary} geometry decoder uses robot state to reduce reconstruction error without extracting the corresponding geometry from images. 
Such shortcuts have been studied in action prediction~\citep{zhao2025statefree}, but their effect on privileged geometry supervision remains underexplored.
\emph{View dominance} occurs when the broad-coverage camera explains much of a fused reconstruction target, leaving the contribution of other views unclear. 
% \TODO{To make this sentence easier to understand, I think you should state before that the three views are head and the two wrists.}
Finally, \emph{task-irrelevant global prediction} can devote much of the \pw{auxiliary reconstruction} objective to background and robot geometry while giving little weight to the local contact and occlusion geometry needed for control. Together, these limitations allow low geometric loss without ensuring that the visual features preserve the geometry relevant to manipulation.
% Together, these limitations \pw{can lead to nearly-}accurate geometric prediction without ensuring that visual features preserve the information needed for manipulation.

% Motivated by these observations, we propose \methodname{}, a training framework for learning geometry-aware representations from multi-view RGB observations. 
% Its central idea is to prevent nonvisual information from solving the geometric learning objective and to align the supervision with the evidence available in the observed camera views. 
% This encourages the visual encoder to capture task-relevant and complementary 3D information while preserving the standard RGB and proprioceptive interface of the underlying policy. 
% Privileged geometry is used only during training, and the framework can be applied to different policy families without introducing additional geometric sensing at deployment.

\pw{To overcome these limitations, w}e propose \methodintro{}, a training framework that directs privileged geometry supervision toward task-relevant visual learning.
% Its core idea is to separate geometric learning from proprioceptive shortcuts, encourage complementary information across camera views, and focus supervision on geometry relevant to manipulation.
%\TODO{Instead of the last sentence, I would explain how each limitation is avoided or reduced in LEAP. solved}
% By aligning geometric supervision with the available visual evidence and the needs of control, \methodname{} encourages policies to learn useful 3D representations while requiring only RGB observations and robot state at deployment.
Its geometry decoder receives only visual features, limiting proprioceptive shortcuts. 
% \TODO{The geometry decoder is what was called geometry predictor, right?
% It would be better to be consistent.}
During training, we randomly mask views and matches a partial geometry target to the remaining views, making wrist contributions easier to identify. 
% Combining this target with full-scene reconstruction gives greater weight to geometry relevant to manipulation. 
At deployment, the policy uses only multi-view RGB images and robot state.

Our contribution\pw{s} can be summarized as follows:
\begin{itemize}
    % \item We argue that an auxiliary geometric loss does not necessarily induce a task-relevant geometric visual representation. We identify two potential causes, including shortcut prediction from correlated robot states and loss domination by predictable robot and background regions.
    \item We identify proprioceptive shortcuts, view dominance, and task-irrelevant global prediction as limitations of privileged geometry supervision for multi-view visuomotor learning.

    % \item Motivated by these observations, we introduce a geometry-supervised training framework that isolates visual learning from nonvisual shortcuts, targets interaction-relevant geometry, and conditions both geometric and action supervision on the available camera views.
    \item We introduce \methodname{}, combining visual-only geometric reconstruction, view dropout with visibility-aligned partial reconstruction, and workspace-constrained supervision.

     % \item We conduct controlled representation and policy evaluations with different RGB-based policies across simulation and real-robot settings, examining geometric generalization, modality shortcuts, and missing-view robustness.
    \item We evaluate these limitations and the corresponding designs through abundant controlled experiments and policy evaluations across simulation and real-robot settings with different RGB-based policy architectures.
   
\end{itemize}

\section{Related Work}
\label{sec:related_work}

\pw{Precise 3D manipulation can be achieved by directly using 3D observations or multiview RGB images as policy inputs, but also via geometric supervision.
Next, we recall those three research directions and discuss related work on learning shortcuts.}
% \TODO{If necessary, I think this section could be shortened.}

% \paragraph{Explicit 3D Inputs for Robotic Manipulation.}
% Explicit 3D observations provide geometric structure for visuomotor control. \citet{shridhar2023peract} voxelize RGB-D observations, \citet{goyal2023rvt} render virtual views from point clouds, and \citet{gervet2023act3d} use sensed depth to lift image features into a 3D feature field for manipulation.
% % \TODO{As stated, it's not clear that the latter work uses 3D observations}
% GNFactor couples scene reconstruction and control through a shared voxel representation~\citep{ze2023gnfactor}. More recent policies encode point clouds for action prediction~\citep{ze2024dp3,wang2024rise}, generate trajectories conditioned on 3D scenes~\citep{ke2024diffuseractor}, or combine structured point maps with RGB observations~\citep{jia2025pointmappolicy}. These methods retain geometric observations or representations in the deployment pipeline. 
% \pw{In contrast, o}ur framework uses geometry as training supervision and executes from RGB observations and proprioception \pw{only}.
\paragraph{Explicit 3D Inputs for Robotic Manipulation.}
\citet{shridhar2023peract} voxelize RGB-D inputs, \citet{goyal2023rvt} render virtual views from point clouds, and \citet{gervet2023act3d} use sensed depth to lift visual features into a 3D field. GNFactor shares a voxel representation between reconstruction and control~\citep{ze2023gnfactor}. Other policies encode point clouds~\citep{ze2024dp3,wang2024rise}, predict trajectories from 3D scenes~\citep{ke2024diffuseractor}, or combine point maps with RGB~\citep{jia2025pointmappolicy}. These approaches retain geometric observations or representations at deployment. LEAP uses geometry only for training supervision and acts from RGB and proprioception.

\paragraph{Multi-View Learning and Camera Availability.}
Multi-view observations provide complementary evidence for manipulation. \citet{hsu2022hands} show that hand-centric and global views differ in their coverage and generalization behavior. MV-MWM \citep{seo2023mvmwm} reconstructs masked views, while 3D-MVP \citep{qian2024threedmvp} and DiffuView \citep{zhang2026diffuview} use multi-view pretraining to learn spatial representations. 
MAD \citep{almuzairee2025mad} learns fused and per-view features for missing-camera robustness, and Multi-Camera View Scaling \citep{xie2026multicamera} uses synchronized views to increase the effective training data. 
\citet{akinola2020multicamera} use camera dropout to make RGB policies robust to missing views.
\pw{In contrast, o}ur concern is the geometric objective under such masking: a low loss on a fused scene may leave the contribution of a wrist view unclear.

\paragraph{Geometric Representation Learning for RGB Policies.}
Geometric supervision can improve RGB policies without requiring 3D observations at inference. DPR and SPA pretrain encoders through depth prediction and multi-view rendering, while eVGGT and GP3 transfer pretrained geometric knowledge~\citep{wang2024dpr,zhu2024spa,vuong2025evggt,qian2025gp3}. \citet{wang2025monolift} distill depth-guided spatial, temporal, and action cues into an RGB policy. During policy learning, methods predict depth, future scenes, or 3D motion~\citep{li2025qdepth,he2026foresight,qian2025geopredict,kim2026pri4r,he2026elan4d}, align features with geometric teachers~\citep{li2025spatialforcing,wang2026vega,wang2026track4action,shi2026thinkvla}, or use training-only 3D Gaussian prediction heads~\citep{jiang2026gaussiandreampp,zhang2026gaussianwam}.

Other methods focus on interaction-relevant structure: object geometry~\citep{ding2026mindvla,liu2026sam3dalignment}, subtask-relevant regions~\citep{gao2026focalvla}, or affordances and goal regions~\citep{zheng2026gift}. \citet{yang2026mvucf} combine depth and cross-view correspondence supervision, while \citet{xu2026actiongeometry} jointly predict actions and future point maps from features that include proprioception. LEAP instead reconstructs workspace and view-matched partial point clouds from visual features alone, encouraging geometric learning from complementary camera views.

\paragraph{Proprioceptive Shortcuts in Visuomotor Learning.}
Prior work identifies proprioceptive over-reliance~\citep{lin2024spawnnet} and optimization-driven visual suppression~\citep{lu2026gap}.
State-free Policy~\citep{zhao2025statefree} removes proprioception to improve spatial generalization, while ReViP~\citep{li2026revip} rebalances vision and proprioception to reduce false task completion.
% Closest to our setting, \citet{patil2026anyviewdex} combine privileged 3D object-coordinate prediction with multi-view alignment, recognizing proprioceptive shortcuts in state-fused distillation.
We instead examine, both theoretically and empirically, how proprioceptive shortcuts can weaken privileged point-cloud supervision in multi-view imitation learning.
To mitigate these shortcuts and focus supervision on task-relevant geometry, \methodname{} reconstructs workspace point clouds from visual features alone and supplements full reconstruction with visibility-aligned partial supervision under wrist-view dropout.
This design encourages complementary local geometric representations while retaining proprioception for action prediction.

\section{Background}
\label{sec:background}

% \paragraph{Visuomotor Imitation Learning.}
% A visuomotor policy consists of a visual encoder and an action predictor, trained jointly on expert demonstrations.
% Let $\mathcal{V}=\{h,l,r\}$ index the head, left-wrist, and right-wrist cameras, and let $I_{\mathcal{V}}=(I_c)_{c\in\mathcal{V}}$ denote their RGB observations.
%\TODO{I think it would have been better to state that you focus on two-arm robots earlier. Solved}
% The proprioceptive state $S$ contains robot measurements such as joint angles and gripper configurations.
% The visual encoder $f_\theta$ extracts per-camera features
% \begin{equation}
%     Z_c=f_\theta(I_c), \qquad c\in\mathcal{V},
% \end{equation}
% which are aggregated into a visual representation $Z$.
% The action predictor uses this representation together with robot state:
% \begin{equation}
%     a\sim\pi_\eta(\cdot\mid Z,S),
% \end{equation}
% where $\theta$ and $\eta$ denote the encoder and action-predictor parameters, respectively, and $a$ denotes an action or an action sequence.
%\TODO{$\theta$ should have been defined the first time it appears. Solved}
% For Diffusion Policy, the action predictor generates action sequences through iterative denoising conditioned on $Z$ and $S$~\citep{chi2023diffusion}.
%\TODO{Since you give an example for sequence of actions, you may want to give an example for single action as well. Solved}

\paragraph{Visuomotor Imitation Learning.}
In this work, we mainly consider imitation learning for two-arm robots equipped with a head camera and two wrist cameras. 
% \TODO{Since you talked about wrist views in the intro, I think this kind of sentence should appear even earlier.}
A visuomotor policy comprises a visual encoder and an action predictor trained on expert demonstrations. 
Let $\mathcal{V}=\pw{\{h,l,r\}}$ index the head, left-wrist, and right-wrist cameras, and let $I_{\mathcal{V}}=(I_c)_\pw{{c\in\mathcal{V}}}$ denote their RGB observations. The proprioceptive state $S$ includes measurements such as joint angles and gripper configurations. 
A visual encoder $f_\theta$, parameterized by $\theta$, extracts features from each view:
% \TODO{There's only one visual encoder that processes all the views, right?
% I would write "A shared visual encoder"}
\begin{equation}
Z_c=f_\theta(I_c), \qquad c\in\mathcal{V}.
\end{equation}
These features are aggregated into a visual representation $Z$. An action predictor $\pi_\eta$, parameterized by $\eta$, conditions on $Z$ and $S$:
\begin{equation}
a\sim\pi_\eta(\cdot\mid Z,S).
\end{equation}
Here, $a$ may denote a single command, such as joint targets and gripper commands for the next control step, or a sequence of such commands. ACT predicts action sequences in chunks~\citep{zhao2023act}, while Diffusion Policy generates them through iterative denoising~\citep{chi2023diffusion}.

\paragraph{Privileged Geometric Supervision.}
Privileged geometric supervision augments imitation learning with a geometric target $G$, such as a point cloud derived from depth measurements, that is available during training but not required at deployment.
An auxiliary decoder $g_\phi$, with parameters $\phi$, reconstructs $G$ from the same visual representation $Z$ used by the action predictor, optionally conditioned on robot state.
% \TODO{I would suggest to choose between "decoder" or "predictor" and be consistent}
We distinguish two reconstruction settings:
\begin{equation}
    \underbrace{g_\phi(Z,S)}_{\text{state-conditioned}},
    \qquad
    \underbrace{g_\phi(Z)}_{\text{visual-only}}.
\end{equation}
The reconstruction loss trains the auxiliary decoder and the shared visual encoder alongside the imitation-learning objective.
% \TODO{In the paper, I think you use different terms to designate "reconstruction loss". 
% I suggest to be consistent as well.}
% \TODO{I think the imitation-learning objective should also be recalled.}
At deployment, the auxiliary decoder is removed, and the policy predicts actions using only RGB observations and proprioceptive state.

\paragraph{Reconstruction Risk.}
We quantify the recoverability of target geometry $G$ from decoder input $X$ using the optimal population reconstruction risk $\mathcal{R}_{d}(G\mid X)$ \pw{defined as follows for}
$d(\widehat G,G)\geq 0$ a divergence that measures error between a prediction and its target:
% \TODO{I think this definition of distortion function is insufficient.
% For instance, you also want $d(G, G) = 0$.
% Is your distortion function a divergence?}
\begin{equation}
    \mathcal{R}_{d}(G\mid X)
    :=
    \inf_g
    \mathbb{E}\!\left[
        d\!\left(g(X),G\right)
    \right],
    \label{eq:optimal_reconstruction_risk}
\end{equation}
where the infimum ranges over all measurable decoders $g$, and the expectation is over the joint distribution of $(G,X)$.
The divergence takes the prediction first and the target second.
We consider $X=Z$ for visual-only reconstruction and $X=(Z,S)$ for state-conditioned reconstruction.
For vector-valued $G$, $\mathcal{R}_{2}$ denotes the risk under squared Euclidean error,
$d_2(\widehat G,G)=\|\widehat G-G\|_2^2$.

\section{Method}
\label{sec:method}
% \TODO{The title of this section mainly fits its second part...}

% \TODO{I find the following paragraph not so clear. 
% Instead I would first discuss the three issues mentioned in the introduction.
% Then, in another paragraph, I would present the proposed method by starting with its high-level description and providing a figure depicting its architecture}
% \TODO{If you want to keep everything under one section, you should introduce its organization first to help the readers understand what they are going to read.}
% % We seek to make training-only privileged 3D supervision an identifiable and policy-relevant constraint on an RGB policy. 
% \TODO{Is "training-only" and "privileged" a bit redundant?}
% However, state-conditioned decoders permit proprioceptive shortcuts, while full-view reconstruction leaves wrist-specific contributions ambiguous when the head view dominates. 
% \methodname{} therefore reconstructs task-relevant partial point clouds solely from visual features and couples wrist dropout with targets supported by each retained head-wrist pair. 
% This aligned supervision encourages complementary wrist geometry in the shared policy representation without requiring geometric input at deployment.

We first identify three limitations of privileged geometric supervision: proprioceptive shortcuts, dominant-view reliance, and reconstruction objectives dominated by task-irrelevant geometry.
These issues can allow low reconstruction error without ensuring that visual features capture the geometry needed for control.
% We then present our framework and its training objective.
%\TODO{Is the previous sentence needed given the next one?
%At least, it seems redundant. Solved}
Guided by this analysis, we introduce \methodname{}, which aims to make geometric supervision an observation-aligned and policy-relevant constraint on visual representations.
As shown in \Cref{fig:method}, its core principles are to separate geometric learning from proprioceptive shortcuts, encourage complementary information across views, and focus supervision on task-relevant geometry.
The resulting policy requires only RGB observations and proprioceptive state at deployment.
% \TODO{I feel that the last sentence is repeated too many times...}
% \Cref{fig:method} provides an overview of the architecture.

\begin{figure}[t]
    \centering
    \includegraphics[width=0.85\linewidth]{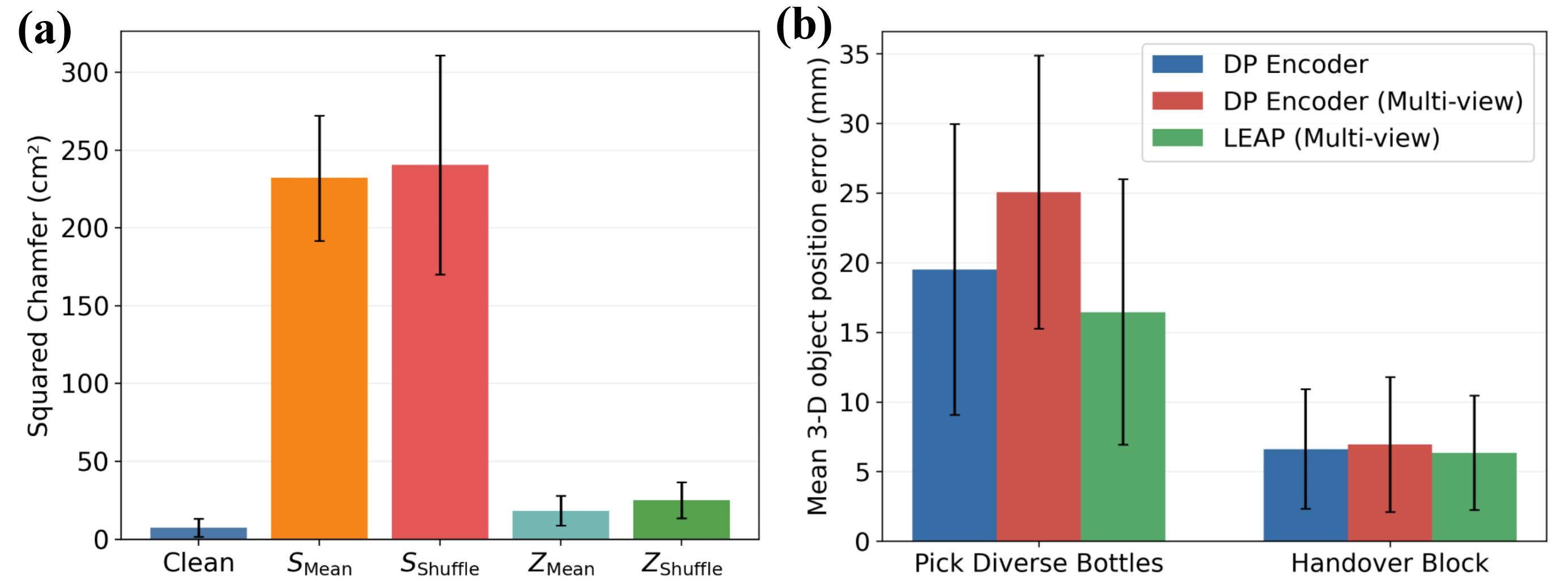}
    \caption{Preliminary diagnostics of geometry learning. (a) Reconstruction error under unperturbed inputs (Clean), mean replacement (Mean), or shuffling (Shuffle) of robot state $S$ or visual features $Z$. (b) Mean 3D object-position error on two tasks, comparing DP Encoder, DP Encoder with Multi-view setting, and LEAP Encoder with Multi-view setting. Error bars denote standard deviations.}
    \label{fig:pre}
    \vspace{-2ex}
\end{figure}

\subsection{When Does Privileged Geometry Benefit Policy Learning?}
\label{sec:when}

Low reconstruction error alone cannot indicate that a visual representation captures the geometry required for control.
We first analyze proprioceptive and dominant-view shortcuts, accompanied by preliminary empirical diagnostics.
We then discuss how global reconstruction objectives can underweight task-relevant local geometry.

\paragraph{Proprioceptive Shortcut.}
Recent methods~\citep{xu2026actiongeometry,yang2026robo3r} predict privileged geometry from both visual features and robot state.
Robot state determines the robot's configuration and may correlate with scene geometry, allowing the decoder to reduce reconstruction error without requiring visual features to encode the corresponding geometry.

We investigate this shortcut using a DP policy trained with state-conditioned geometry reconstruction on RoboTwin.
% \TODO{Remark: the best would have been to use a policy trained in some previous works.}
With the encoder and decoder frozen, we replace either robot state or visual features with their training-set mean
($S_{\mathrm{Mean}}$, $Z_{\mathrm{Mean}}$)
or shuffle them across evaluation samples
($S_{\mathrm{Shuffle}}$, $Z_{\mathrm{Shuffle}}$),
keeping the other input and reconstruction target unchanged.
We measure reconstruction quality using squared Chamfer error, which averages mean squared nearest-neighbor errors between predicted and target point clouds.
% \TODO{Why do you need to specify "aggregates"?
% This suggests that there are many means...
% Why not simply say: squared Chamfer distance=average squared mismatch between two 3D point clouds (if you don't want to provide the formula)?}
As shown in \Cref{fig:pre}(a), both state interventions sharply increase reconstruction error, whereas the corresponding visual interventions produce much smaller increases.
This difference suggests that the trained decoder relies strongly on proprioception, although it does not establish that visual features are completely invalid.

To explain why the reconstruction objective permits this behavior, the following proposition compares the optimal population risks with and without state conditioning (see \Cref{app:state_nonidentifiability} for proof).  

\begin{proposition}[State-induced non-identifiability]
\label{prop:state_nonidentifiability}
Let $d:\mathcal{G}\times\mathcal{G}\to[0,\infty)$ be a measurable reconstruction cost; no metric properties are required.
% \TODO{Why is it important to always specify "prediction first and target second"?
% If you keep your notations consistent, isn't it enough?}
For the population risk in \Cref{eq:optimal_reconstruction_risk}, with the infimum taken over all measurable decoders,
% \TODO{Instead of taking the inf over all measurable decoders, I suggest to do it over a hypothesis class of measurable decoders.
% I think your result is true for any such hypothesis class, which makes it stronger.
% Indeed, if the hypothesis class is induced by NNs, this directly applies to your experiments.}
\begin{equation}
    \mathcal{R}_{d}(G\mid Z,S)
    \leq
    \mathcal{R}_{d}(G\mid Z).
    \label{eq:risk_monotonicity}
\end{equation}
For $\mathcal{G}=\mathbb{R}^{m}$,
$\mathbb{E}[\|G\|_2^2]<\infty$, and squared Euclidean divergence
$d_2(\widehat G,G)=\|\widehat G-G\|_2^2$,
\begin{equation}
\begin{aligned}
    &\mathcal{R}_{2}(G\mid Z)
    -
    \mathcal{R}_{2}(G\mid Z,S) 
     =
    \mathbb{E}\!\left[
        \left\|
            \mathbb{E}[G\mid Z,S]
            -
            \mathbb{E}[G\mid Z]
        \right\|_2^2
    \right].
\end{aligned}
\label{eq:state_risk_gap}
\end{equation}
\end{proposition}

Here, $\mathbb{E}[G\mid Z,S]$ and $\mathbb{E}[G\mid Z]$ are conditional means under the joint data distribution: the population-optimal predictions under squared Euclidean loss, rather than outputs of two trained decoders.
The general inequality applies to squared Chamfer \pw{error}, while the risk-gap identity concerns vector-valued reconstruction under squared Euclidean error.
%The proof is provided in \Cref{app:state_nonidentifiability}.

\pw{This} limitation \pw{can occur} even with perfect state-conditioned reconstruction.
In the squared Euclidean setting, if $G=\psi(S)$ is not almost surely constant and $Z$ is constant, then \pw{we can have}
$\mathcal{R}_{2}(G\mid Z,S)=0$, whereas
\begin{equation}
    \mathcal{R}_{2}(G\mid Z)
    =
    \mathbb{E}\!\left[
        \|G-\mathbb{E}[G]\|_2^2
    \right]
    >0.
\end{equation}
Thus, low state-conditioned reconstruction error alone cannot certify that geometry is recoverable from visual features.

\paragraph{Dominant-View Shortcut.}
Even when reconstruction uses only visual inputs, the representation may rely primarily on a subset of views.
The head camera provides broad scene coverage, while wrist cameras reveal complementary local geometry.
If the head view already explains most reconstruction target, low reconstruction error cannot indicate the contribution of wrist observations.

We examine whether wrist features improve object-position prediction using frozen visual representations.
For the DP baseline trained with full point-cloud reconstruction, \emph{DP Encoder} uses only head-view features $Z_h$, whereas \emph{DP Encoder (Multi-view)} uses concatenated features $[Z_h,Z_l,Z_r]$ from the head and both wrist cameras.
We additionally evaluate \emph{\methodname{} (Multi-view)}, which uses the same three-view feature configuration but with representations learned through \methodname{}.
In all settings, camera-specific features are extracted before fusion and kept frozen while a small MLP is trained to predict task-object 3D positions.
We report mean Euclidean error on held-out episodes over three random seeds for MLP training.
As shown in \Cref{fig:pre}(b), adding wrist features to the full-reconstruction DP baseline does not reduce mean prediction error on either task, whereas \methodname{} achieves lower mean error using the same three-view feature configuration.
These results suggest that simply including wrist features does not guarantee better object-position prediction, motivating supervision that encourages complementary geometric information across views.
% We examine whether wrist features improve object-position prediction from a frozen encoder trained to reconstruct the full point cloud from all three camera views.
% % \TODO{Is readout a commonly-used term?
% % Why not use a more explicit term like prediction?}
% Small MLP predicts task-object 3D positions from either head-view features or concatenated head and both wrist-view features, extracted before fusion.
% % \TODO{Why do you use the term "probe"?}
% We evaluate mean Euclidean error on held-out episodes using three probe seeds.
% As shown in \Cref{fig:pre}(b), adding wrist features to the full-reconstruction baseline does not reduce mean prediction error on either task and increases it in some cases.
% These results suggest that simply adding wrist features does not effectively exploit the complementary information available from the wrist views.
% \TODO{I think it would be better to keep consistent: distance vs error}
% Thus, wrist features provide no additional benefit for object-position prediction under this probing protocol.
% This result does not directly establish view dominance after fusion or isolate geometry visible only from the wrists.
% \TODO{What's the purpose of this last sentence?}
% \TODO{This whole paragraph seems to come from an LLM.
% It would be better to adapt the explanation to the paper.}
The following proposition formalizes a related limitation of the reconstruction objective: when the head view alone permits accurate reconstruction, low error cannot establish that wrist observations contribute.

\begin{proposition}[Non-identifiability of view-wise contributions]
\label{prop:dominant_view}
Let $I_h$, $I_l$, and $I_r$ denote observations from the head, left-wrist, and right-wrist cameras, respectively, and let $G$ denote the target geometry.
Suppose a head-view encoder $f_h$ and decoder $g_h$ satisfy
\[
    \mathbb{E}\!\left[
        d\!\left(g_h(f_h(I_h)),G\right)
    \right]
    \leq \varepsilon.
\]
Consider full-view models with concatenated embeddings. If their hypothesis classes admit constant wrist encoders and a decoder that applies $g_h$ to the head embedding alone, then there exists an admissible pair $(F,g)$ satisfying \[  F(I_h,I_l,I_r) = [f_h(I_h),c_l,c_r],\quad \mathbb{E}\!\left[d\!\left(g(F(I_h,I_l,I_r)),G\right)\right] \leq \varepsilon, \]
% \TODO{It would be better if this sentence could be stated formally?}
% \TODO{I'm not sure how you did the proof, but it seems that the full-view model necessarily includes the previous equation, since the decoder for the full-view acts on the concatenation of the embedding of each view.
% Therefore, it is possible for this decoder to ignore the wrist views.}
\end{proposition}

The proof \pw{(see \Cref{app:dominant_view})} constructs a full-view model that reproduces the head-only predictor while ignoring both wrist observations.
Consequently, reconstruction accuracy alone cannot certify wrist-view use.
This matters for control because geometric coverage and control relevance need not coincide.
A wrist view may reveal a gripper--bottle gap hidden from the head camera: errors in this local geometry may contribute little to reconstruction error while determining grasp success.

% \paragraph{Task Relevance of Geometric Supervision.}
% The spatial distribution of supervision also matters.
% Consider a dense depth target whose task-relevant region occupies a fraction $\alpha\in(0,1)$ of the supervised pixels.
% Its uniformly averaged per-pixel reconstruction loss decomposes as
% \begin{equation}
%     \mathcal{L}_{\mathrm{global}}
%     =
%     \alpha\,\mathcal{L}_{\mathrm{relevant}}
%     +
%     (1-\alpha)\,\mathcal{L}_{\mathrm{remaining}},
% \label{eq:global_reconstruction_decomposition}
% \end{equation}
% where the two terms are the mean errors within their respective regions.
% \TODO{Actually, $\mathcal{L}_{\mathrm{relevant}}$ is already $\alpha$ of $\mathcal{L}_{\mathrm{global}}$.
% Scaling by $\alpha$ reduces even more its weight, no?}
% When $\alpha$ is small, accurate prediction of the remaining regions can yield low global error despite substantial errors in task-relevant geometry.
% This motivates explicitly constraining critical local geometry alongside global reconstruction.

\paragraph{Task Relevance of Geometric Supervision.}
Uniform spatial averaging can dilute supervision of task-relevant geometry.
Consider a depth target with $N$ pixels and a task-relevant subset $\mathcal{R}$ occupying a fraction $\alpha=|\mathcal{R}|/N$.
For per-pixel reconstruction errors $\ell_i$, define the regional means as
\[
    \mathcal{L}_{\mathrm{relevant}}
    =
    \frac{1}{|\mathcal{R}|}\sum_{i\in\mathcal{R}}\ell_i,
    \qquad
    \mathcal{L}_{\mathrm{remaining}}
    =
    \frac{1}{|\mathcal{R}^{c}|}\sum_{i\in\mathcal{R}^{c}}\ell_i.
\]
The global average therefore satisfies
\begin{equation}
    \mathcal{L}_{\mathrm{global}}
    =
    \alpha\,\mathcal{L}_{\mathrm{relevant}}
    +
    (1-\alpha)\,\mathcal{L}_{\mathrm{remaining}}.
    \label{eq:global_reconstruction_decomposition}
\end{equation}
For small $\alpha$, low global error can mask large task-relevant errors, motivating local supervision alongside global reconstruction.

The propositions identify possible, not inevitable, shortcuts.
Together with our diagnostics and task-relevance analysis, they motivate supervision that promotes visual geometric encoding, complementary view contributions, and emphasis on control-relevant regions.

% \TODO{How do we know which part is relevant?}

\begin{figure}[t]
    \centering
    \includegraphics[width=0.99\linewidth]{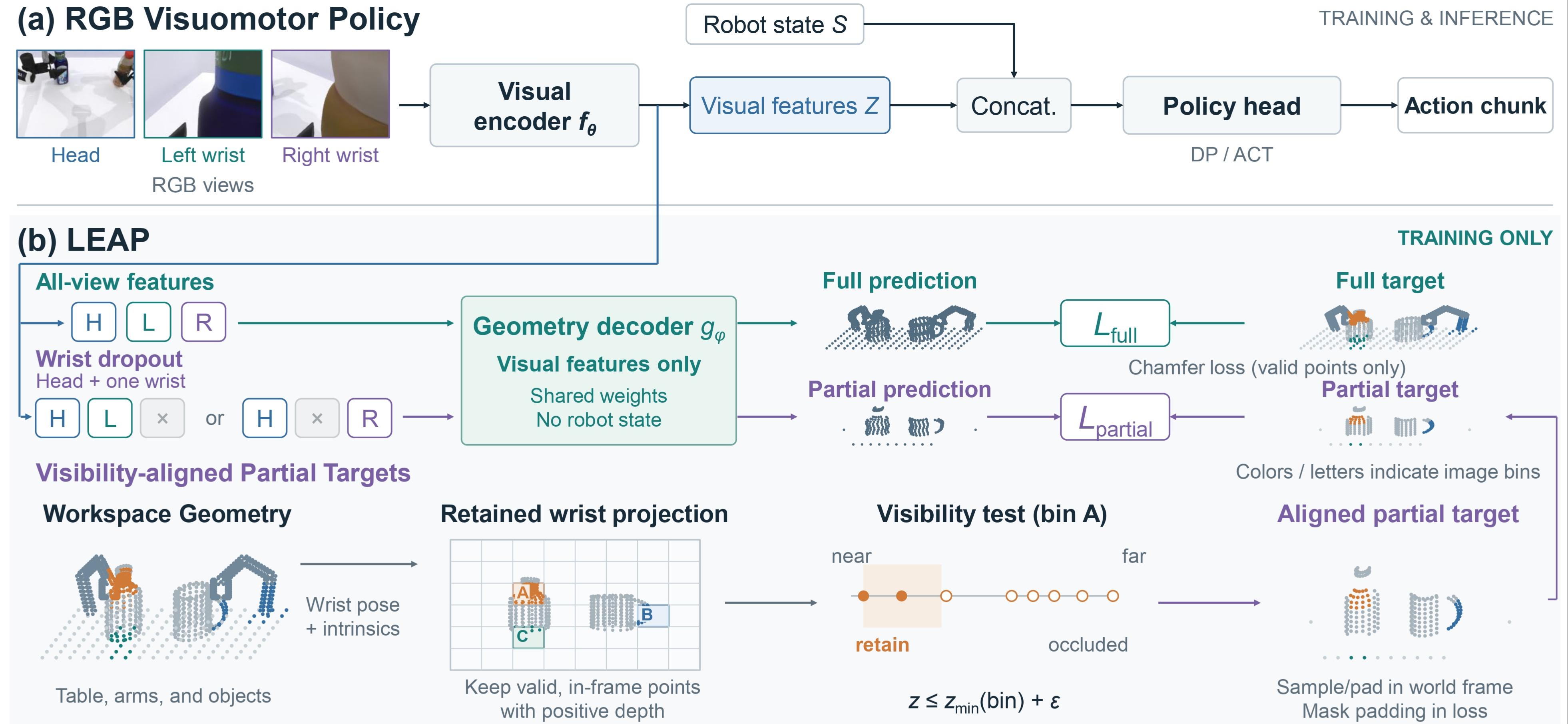}
    \caption{Overview of \methodname{}. (a) The workflow of an RGB-based policy. (b) During training, \methodname{} supervises the visual encoder through full and partial point-cloud reconstruction, using visual features alone. Partial reconstruction retains the head and one randomly selected wrist, with targets matched to that wrist’s approximate visibility. The geometry decoder is removed at inference.
}
    \label{fig:method}
    \vspace{-2ex}
\end{figure}

\subsection{\methodintro{}}
\label{sec:leap}

We introduce \methodname{} to address the limitations identified in \Cref{sec:when}.
The method combines grounding geometry in visual representations, learning complementary geometry with visibility-aligned partial reconstruction, and workspace-constrained geometric supervision.
% \TODO{Why not list the techniques in the same order as the limitations above?}
% Full reconstruction provides the base geometric objective, while partial reconstruction adds supervision for local geometry observed by the retained wrist.

\paragraph{Grounding Geometry in Visual Representations.}
State-conditioned reconstruction can exploit proprioceptive information without requiring the corresponding geometry in visual features (\Cref{prop:state_nonidentifiability}).
We therefore allow the action predictor to use robot state but exclude it from the geometry decoder.
% Let $\mathcal{V}=\{h,l,r\}$ index the head, left-wrist, and right-wrist cameras.
% For each camera $c\in\mathcal{V}$, the visual encoder extracts features
% $Z_c=f_\theta(I_c)$ from its RGB observation $I_c$.
% \TODO{The last two sentences repeat what the background section recalled, no?}
The reconstruction pathway receives only visual features, and its losses update the encoder shared with the policy.
This removes direct proprioceptive access from geometric reconstruction while preserving its use for action prediction.

% \paragraph{View Dropout and Partial Geometry Alignment.}
% \TODO{Why do you explain this technique last?}
% \TODO{What is an approximate visibility operator? Solved}

% Specifically, we project the world-coordinate points into the retained wrist camera using its intrinsics and extrinsics.
% \TODO{What are intrinsics and extrinsics? Solved}
% We discard nonfinite points, points behind the camera, and points projecting outside the image.
% The projected image is divided into $8\times8$-pixel bins.
% Within each bin, we retain points whose camera depth is within $0.01\,\mathrm{m}$ of the nearest projected point.
% This approximates occlusion using the sampled cloud, without comparison to the original wrist depth image.
% \TODO{I don't understand this last sentence. Solved}
% % Selected points remain in world coordinates and are sampled and padded to form fixed-size targets with validity masks.
% \TODO{More explanation is probably needed here. 
% I suggest to explain formally with some equations. Solved}

% When the left wrist is dropped, the partial objective uses the right-visible target, and vice versa.
% The head view remains available as context, while the target is selected according to the retained wrist.
% This adds a local geometric constraint alongside full reconstruction and encourages the representation to retain information useful for predicting wrist-visible surfaces.

\paragraph{Learning Complementary Geometry.}
The head camera provides broad scene coverage, whereas wrist cameras reveal local geometry that may be obscured from the head view.
To encourage each wrist to contribute complementary geometric information, we apply view dropout during training, always retaining the head view and randomly retaining one wrist:
\begin{equation}
    M\in
    \bigl\{
        \{h,l\},\{h,r\}
    \bigr\}.
    \label{eq:wrist_view_masks}
\end{equation}
Each wrist is thus trained as the sole wrist-side observation accompanying the head camera, encouraging the encoder to exploit its local observations without relying on the other wrist.

\paragraph{Visibility-Aligned Partial Reconstruction.}
To couple view dropout with local geometric supervision, we supplement full-scene reconstruction with a partial objective targeting surfaces approximately visible from the retained wrist.
Let $G^{\mathrm{full}}$ denote the fused point cloud within the manipulation workspace, whose construction is described below, and let $w(M)\in\{l,r\}$ denote the retained wrist.
We construct its partial target as
\begin{equation}
    G_M^{\mathrm{partial}}
    =
    \widetilde{\operatorname{Vis}}_{w(M)}
    \!\left(G^{\mathrm{full}}\right),
    \label{eq:view_aligned_target}
\end{equation}
where $\widetilde{\operatorname{Vis}}_{w(M)}$ selects approximately visible points through projection and depth filtering.
Specifically, we transform the full-cloud points from world coordinates into the retained camera's coordinate frame using its position and orientation, then project them into the image using its calibrated projection model.
After removing invalid points and those behind the camera or outside the image, we group nearby projections into image bins.
Within each bin, we retain points close in depth to the nearest point and treat more distant points as occluded.
Visibility is approximate because occlusion is inferred from the sampled cloud.
The selected points retain their world coordinates and are sampled or padded to a fixed size, with validity masks excluding padded entries from the reconstruction loss.
The partial target therefore follows the retained wrist, while the head view remains available as context.
Together with full-scene reconstruction, this objective encourages the representation to capture local wrist-visible geometry.

\paragraph{Workspace-Constrained Geometric Supervision.}
We construct the full reconstruction target from depth-derived point clouds captured by all three cameras.
Camera-space points are transformed into a common world coordinate frame, cropped to the configured manipulation workspace, and merged.
% \TODO{How do we perform the cropping?}
Farthest-point sampling produces the full target $G^{\mathrm{full}}$ with up to a fixed number of valid points; If the number of points is insufficient, we pad the target to the fixed number.
% \TODO{The last two sentences are not so clear.}
Geometric supervision uses the world-coordinate XYZ values.
% \TODO{I think this part needs more details}

The workspace crop provides a spatial relevance constraint by excluding geometry outside the manipulation region.
Point-cloud targets make this selection explicit in 3D and place observations from different cameras in a common coordinate frame.
Here, full reconstruction refers to the fused target within this workspace.

\paragraph{Training Objective.}
Let $\widehat{G}^{\mathrm{full}}$ and
$\widehat{G}_M^{\mathrm{partial}}$ denote the full and partial predictions.
We supervise both using Chamfer reconstruction loss:
\begin{equation}
\begin{aligned}
    \mathcal{L}_{\mathrm{geo}}
    =
    \mathbb{E}_{M}\Big[
        \lambda_{\mathrm{full}}\,
        \ell_{\mathrm{CD}}\!\left(
            \widehat{G}^{\mathrm{full}},G^{\mathrm{full}}
        \right) 
        +
        \lambda_{\mathrm{partial}}\,
        \ell_{\mathrm{CD}}\!\left(
            \widehat{G}_M^{\mathrm{partial}},
            G_M^{\mathrm{partial}}
        \right)
    \Big].
\end{aligned}
\label{eq:leap_geometry_objective}
\end{equation}
Here, $\lambda_{\mathrm{full}},\lambda_{\mathrm{partial}}\geq 0$ weight the two reconstruction terms, and $\mathbb{E}_M$ averages over the wrist-view sampling rule.
The loss $\ell_{\mathrm{CD}}$ is evaluated with the validity and frame masks.
% \TODO{valid-point masks = validity masks?
% It would be better to stay consistent. Solved}
Padded points are excluded from distance computation.
In practice, we specify a minimum number of valid points for full and partial targets separately, excluding targets below their respective thresholds from the corresponding reconstruction losses.
% Full targets require at least one valid point, and partial targets require at least eight; invalid targets do not contribute to their respective losses.
% \TODO{Those seem to be some ad-hoc chosen hyperparameter values. 
% Could they be justified? Solved}
% Here, $\mathbb{E}_M$ averages over the wrist-view sampling rule. Validity masks exclude padded points from Chamfer-distance computation, while frame masks omit empty full targets and partial targets with insufficient valid points from their respective reconstruction losses.

The complete training objective combines action prediction with both geometric terms:
\begin{equation}
    \mathcal{L}_{\mathrm{\methodname{}}}
    =
    \mathcal{L}_{\mathrm{policy}}
    +
    \mathcal{L}_{\mathrm{geo}}.
\label{eq:leap_training_objective}
\end{equation}
% \TODO{Why only use $\lambda_{\mathrm{geo}}$ here? 
% Above you used two weights.
% I think it would be better to be consistent.}
% Here, $\lambda_{\mathrm{geo}}\geq 0$ weights geometric supervision relative to action prediction.
% Both objectives update the shared visual encoder. 
The policy loss and both reconstruction terms jointly update the shared visual encoder.
At deployment, the policy uses RGB observations and robot state, without privileged point-cloud targets or the auxiliary reconstruction pathway.

\section{Experimental Results}

\subsection{Experimental Setup}
\label{sec:experimental_setup}

\paragraph{Datasets and Tasks.}
We evaluate \methodname{} on 50 RoboTwin tasks, six ManiSkill tasks, and three real-world manipulation tasks.
To examine performance across different geometric requirements, we group the RoboTwin tasks into four categories:
\textbf{Object Spatial} (14),
\textbf{Multi-object} (14),
\textbf{Interaction} (14), and
\textbf{Less-geometry} (8).
We report success rates for each category and the overall average across all 50 tasks.
The ManiSkill and real-world experiments provide complementary evaluations across simulation environments and physical manipulation settings. More details in \Cref{subsec:task}.

\paragraph{Baseline Methods.}
We compare against Action Chunking with Transformers (ACT)~\citep{zhao2023act} and Diffusion Policy (DP)~\citep{chi2023diffusion} as standard RGB-based imitation learning baselines.
We also include eVGGT~\citep{vuong2025evggt}, GAP~\citep{xu2026actiongeometry}, and Spatial Forcing~\citep{li2025spatialforcing} as geometry-aware approaches; DP3~\citep{ze2024dp3}, which directly consumes point-cloud, as a geometry-input upper bound; and $\pi_0$~\citep{black2025pi0} and RDT~\citep{liu2025rdt} as robot foundation models.
For the DP-based comparisons, we integrate the eVGGT visual encoder and adapt Spatial Forcing's representation-alignment strategy to DP, denoting these variants as eVGGT-DP and Spatial Forcing-DP.
We apply \methodname{} to both policy architectures, yielding DP-\methodname{} and ACT-\methodname{}, to evaluate its effectiveness across different action prediction frameworks.

\paragraph{Training and Evaluation.}
Our RGB policies use observations from three cameras, including one head camera and two wrist cameras, together with robot proprioceptive state.
We train DP-based policies for 600 epochs and ACT-based policies for 4,000 epochs, using three independent random seeds.
For \methodname{}, privileged point-cloud supervision is used only during training; inference requires only RGB observations and proprioception.
We summarize performance using task success rates, reporting the mean and standard deviation across seeds. More implementation details in \Cref{sec:imp}.

\begin{table}[t]
\centering
\caption{Overall performance on the 50 RoboTwin tasks across 3 random seeds. Tasks are grouped according to their geometric requirements. DP, RDT, $\pi_0$ use the official results with no SD.}
\label{tab:robotwin_main}

\scriptsize
\setlength{\tabcolsep}{3.0pt} % Slightly reduced padding further to fit extra cols
\begin{tabular}{lccccccc}
\toprule

\multirow{2}{*}{\textbf{Method}} &
\multirow{2}{*}{\textbf{Obs.}}&
\multicolumn{5}{c}{\textbf{Success Rate (\%) by Task Category} $\uparrow$} & 
\multirow{2}{*}{\textbf{Params.} $\downarrow$}\\
\cmidrule(lr){3-7}
& &
\textbf{Object Spatial}(14) &
\textbf{Multi-object}(14) &
\textbf{Interaction}(14) &
\textbf{Less-geometry}(8) & 
\textbf{Average} & \\
\midrule
DP3 (RSS-24) & P.C.
& 45.3$\pm$3.2 & 41.6$\pm$2.2 & 66.9$\pm$5.5 & 80.4$\pm$4.3 & 55.9$\pm$3.7 & 262.43M\\

\midrule
DP (RSS-23) & RGB(1)
& 19.8 & 24.2 & 30.4 & 45.0 & 28.0 & 85.62M\\

DP-MultiView (RSS-23) & RGB
& 18.0$\pm$2.6 & 19.5$\pm$4.9 & 30.2$\pm$4.7 & 44.9$\pm$3.5 & 26.1$\pm$4.0 & 86.35M\\

DP-eVGGT (arxiv-25) & RGB
& 25.1$\pm$3.0 & 25.5$\pm$4.2 & 38.1$\pm$3.7 & 51.3$\pm$4.1 & 33.0$\pm$3.8 & 129.66M\\

RDT (ICLR-25) & RGB
& 26.4 & 21.0 & 46.4 & 51.5 & 34.5 & 1200M\\

DP-SF (ICLR-26) & RGB
& 35.6$\pm$2.0 & 32.5$\pm$2.8 & 50.0$\pm$4.7 & 60.8$\pm$5.1 & 42.8$\pm$3.5 & 129.66M\\

$\pi_0$ (RSS-25) & RGB
& 34.3 & 38.9 & 52.1 & 70.9 & 46.4 & 3300M\\

GAP (CVPR-26) & RGB
& 39.9$\pm$2.5 & 38.9$\pm$4.3 & 54.6$\pm$3.6 & 62.8$\pm$4.9 & 47.4$\pm$3.7 & 1200M\\

\textbf{DP-\methodname{} (Ours)} & RGB
& \textbf{44.1$\pm$2.6} & \textbf{47.1$\pm$3.6} & \textbf{68.8$\pm$3.8} & \textbf{73.4$\pm$4.2} & \textbf{56.5$\pm$3.5} & 87.78M\\

\midrule
ACT-pointcloud & P.C.
& 32.9$\pm$2.4 & 29.3$\pm$2.9 & 58.6$\pm$3.3 & 79.1$\pm$3.6 & 46.5$\pm$3.0 & 262.84M\\

\midrule

ACT (RSS-23) & RGB
& 16.6$\pm$3.0 & 18.8$\pm$3.4 & 38.9$\pm$4.2 & 50.4$\pm$6.9 & 28.9$\pm$4.1 & 83.90M\\

\textbf{ACT-\methodname{} (Ours)} & RGB
& 36.8$\pm$3.5 & 33.2$\pm$2.6 & 58.8$\pm$4.4 & 75.5$\pm$4.0 & 48.1$\pm$3.6 & 85.63M\\

\bottomrule
\end{tabular}

\end{table}

\subsection{Performance Analysis}

\paragraph{Quantitative Evaluations on RoboTwin.}
As shown in \Cref{tab:robotwin_main}, \methodname{} improves both backbones across all four task categories, raising average success from 28.0\% to 56.5\% for DP and from 28.9\% to 48.1\% for ACT, with only 2.5\% and 2.1\% more parameters, respectively. Among RGB policies, DP-\methodname{} achieves the highest average success (56.5\%), outperforming geometry-aware baselines DP-eVGGT (33.0\%), DP-SF (42.8\%), and GAP (47.4\%), as well as RDT (34.5\%) and $\pi_0$ (46.4\%). Compared with GAP, the strongest competing RGB baseline, its advantage holds across all categories and is largest on Interaction tasks (68.8\% vs.\ 54.6\%).

We use DP3 as an empirical upper-bound reference because it directly receives point clouds at inference.
Despite requiring no point-cloud input at deployment, DP-\methodname{} surpasses this reference in average success rate, achieving 56.5\% versus 55.9\%.
ACT-\methodname{} likewise exceeds its point-cloud counterpart, achieving 48.1\% compared with 46.5\% for ACT-pointcloud.
These results show that our training-only geometry supervision enables RGB policies to reach, and slightly surpass, the average performance of their counterparts with direct access to 3D observations.

\begin{table}[t]
    \centering
    \caption{Success Rate(\%) on the 6 Maniskill tasks across 3 random seeds.}
    \label{tab:maniskill_main}
    \footnotesize
    \setlength{\tabcolsep}{4pt}
    \renewcommand{\arraystretch}{1.15}
    \begin{tabular}{@{}lccccccc@{}}
        \toprule
        Method & PickCube & StackCube &
        \shortstack{Peg\\Insertion} &
        \shortstack{Plug\\Charger} &
        \shortstack{PullCube\\Tool} & PushT & Average \\
        \midrule
        DP3 & 45 $\pm$ 2 & 71 $\pm$ 2 & 7 $\pm$ 3 &
        13 $\pm$ 2 & 91 $\pm$ 4 & 34 $\pm$ 2 & 43.3 $\pm$ 1.5 \\
        \midrule
        DP & 32 $\pm$ 4 & 41 $\pm$ 1 & 0 $\pm$ 1 &
        0 $\pm$ 0 & 68 $\pm$ 3 & 18 $\pm$ 3 & 26.4 $\pm$ 1.7 \\
        DP-SF & 43 $\pm$ 3 & 55 $\pm$ 5 & \textbf{6} $\pm$ 1 &
        6 $\pm$ 1 & 85 $\pm$ 2 & 21 $\pm$ 1 & 35.9 $\pm$ 1.2 \\
        RDT & 40 $\pm$ 2 & 61 $\pm$ 4 & 4 $\pm$ 1 &
        8 $\pm$ 4 & 88 $\pm$ 3 & 20 $\pm$ 4 & 36.8 $\pm$ 1.3 \\
        DP-\methodname{} & \textbf{49} $\pm$ 2 & \textbf{66} $\pm$ 4 &
        \textbf{6} $\pm$ 2 & \textbf{9} $\pm$ 2 &
        \textbf{96} $\pm$ 2 & \textbf{31} $\pm$ 2 &
        \textbf{42.5} $\pm$ 1.1 \\
        \midrule
        ACT-pointcloud & 41 $\pm$ 2 & 57 $\pm$ 6 & 4 $\pm$ 1 &
        5 $\pm$ 2 & 63 $\pm$ 5 & 34 $\pm$ 4 & 34.1 $\pm$ 1.4 \\
        \midrule
        ACT & 19 $\pm$ 1 & 30 $\pm$ 2 & 0 $\pm$ 1 &
        1 $\pm$ 1 & 45 $\pm$ 2 & 25 $\pm$ 2 & 20.0 $\pm$ 1.3 \\
        ACT-\methodname{} & \textbf{40} $\pm$ 2 & \textbf{50} $\pm$ 1 &
        \textbf{4} $\pm$ 1 & \textbf{8} $\pm$ 1 &
        \textbf{69} $\pm$ 3 & \textbf{28} $\pm$ 4 &
        \textbf{33.2} $\pm$ 1.7 \\
        \bottomrule
    \end{tabular}
\end{table}

\paragraph{Quantitative Evaluations on Maniskill.}
As shown in \Cref{tab:maniskill_main}, \methodname{} improves both policy backbones across all six ManiSkill tasks.
Average success increases from 26.4\% to 42.5\% for DP and from 20.0\% to 33.2\% for ACT.
DP-\methodname{} achieves the highest average success among the evaluated RGB policies, outperforming DP-SF (35.9\%) and RDT (36.8\%) while matching or exceeding their success rates on every task.
These improvements substantially narrow the gap to the point-cloud policies, which serve as empirical upper-bound references with direct access to 3D observations.
DP-\methodname{} approaches DP3 in average success (42.5\% versus 43.3\%), while ACT-\methodname{} similarly approaches ACT-pointcloud (33.2\% versus 34.1\%).
Both variants also exceed their point-cloud counterparts on PullCubeTool, demonstrating competitive manipulation performance without requiring point clouds at deployment.

\begin{table}[h]
    % \captionsetup{
    %     font=small,
    %     skip=4pt,
    %     justification=raggedright,
    %     singlelinecheck=false
    % }
    \centering
    \begin{minipage}[t]{0.53\linewidth}
        \vspace{0pt}
        \caption{Ablation: mean success (\%) over five tasks.
        Task-level results appear in Table~\ref{tab:ablation-full}.}
        \label{tab:ablation}
        \small
        \setlength{\tabcolsep}{3pt}
        \renewcommand{\arraystretch}{1.10}
        \begin{tabular*}{\linewidth}{
            @{\extracolsep{\fill}}lrr@{}
        }
            \toprule
            LEAP variant & DP & ACT \\
            \midrule
            Full
                & \textbf{76.6} & \textbf{49.0} \\
            w/o wrist dropout
                & 55.2 & 38.6 \\
            w/ state in decoder
                & 58.2 & 36.8 \\
            w/ depth reconstruction
                & 71.0 & 45.4 \\
            w/o target alignment
                & 62.0 & 44.4 \\
            \bottomrule
        \end{tabular*}
    \end{minipage}\hfill
    \begin{minipage}[t]{0.44\linewidth}
        \vspace{0pt}
        \caption{Real-world success rates (\%).
        Bold marks the better RGB result.}
        \label{tab:realworld}
        \small
        \setlength{\tabcolsep}{2.3pt}
        \renewcommand{\arraystretch}{1.10}
        \begin{tabular*}{\linewidth}{
            @{\extracolsep{\fill}}lrrr@{}
        }
            \toprule
            Task & DP3 & DP & DP-LEAP \\
            \midrule
            Pick-Place
                & 95 & 35 & \textbf{90} \\
            Assemble
                & 70 & 15 & \textbf{85} \\
            Handover-Block
                & 30 & 0 & \textbf{30} \\
            \midrule
            Average
                & 65.0 & 16.7 & \textbf{68.3} \\
            \bottomrule
        \end{tabular*}
    \end{minipage}
\end{table}

% \begin{table}[t]
%     \centering
%     \caption{Success rates (\%) on three real-world manipulation tasks.}
%     \label{tab:realworld}
%     \small
%     \setlength{\tabcolsep}{6pt}
%     \renewcommand{\arraystretch}{1.12}
%     \begin{tabular}{@{}lcccc@{}}
%         \toprule
%         Method &
%         \shortstack{Pick\\Place} &
%         \shortstack{Assemble} &
%         \shortstack{Handover\\Block} &
%         Average \\
%         \midrule
%         DP3       & 95 & 70 & 15 & 58.3 \\
%         \midrule
%         DP        & 35 & 15 & 0  & 16.7 \\
%         DP-\methodname{} &
%         \textbf{90} & \textbf{85} & \textbf{15} & \textbf{63.3} \\
%         \bottomrule
%     \end{tabular}
% \end{table}
% \paragraph{Real-world Experiments.}
% As shown in \Cref{tab:realworld}, \methodname{} improves DP across all three real-world tasks, increasing average success from 16.7\% to 68.3\%.
% The largest improvement occurs on Assemble, where success rises from 15\% to 85\%.
% Without requiring point clouds at deployment, DP-\methodname{} also surpasses DP3 in average success (68.3\% versus 65.0\%).
% It achieves 90\% on PickPlace, compared with DP3's 95\% and achieves higher success on Assemble (85\% versus 70\%).
% HandoverBlock remains challenging, with our method matching DP3 at 30\%. Real-world demos are in our supplementary materials.

\paragraph{Real-world Experiments.}
As shown in \Cref{tab:realworld}, \methodname{} improves DP across all three real-world tasks, raising average success from 16.7\% to 68.3\%. The largest gain occurs on Assemble, from 15\% to 85\%. Without point clouds at deployment, DP-\methodname{} also slightly surpasses DP3 on average (68.3\% vs.\ 65.0\%), achieving 90\% vs.\ 95\% on Pick-Place, 85\% vs.\ 70\% on Assemble, and matching DP3 at 30\% on the challenging Handover-Block task. 
% Real-world demonstrations are provided in the supplementary materials.

\paragraph{Ablation Study.}
\Cref{tab:ablation} shows that each design contributes to average performance.
Removing wrist dropout reduces average success from 76.6\% to 55.2\% for DP and from 49.0\% to 38.6\% for ACT.
Retaining wrist dropout but removing target alignment yields 62.0\% and 44.4\%, respectively, demonstrating the additional benefit of matching geometric supervision to the available views.
Adding robot state to the geometry decoder lowers success to 58.2\% and 36.8\%, supporting state-separated visual reconstruction.
Replacing point-cloud reconstruction with depth reconstruction yields 71.0\% for DP and 45.4\% for ACT, indicating that point-cloud supervision provides a more effective training signal.
See full results in \Cref{subsec:full-a} and sensitivity analysis in \Cref{subsec:sen}.

\section{Conclusion and Limitation}

We presented \methodname{}, which improves RGB visuomotor learning through visual-only reconstruction, workspace-constrained targets, and wrist-view dropout with aligned partial supervision alongside full reconstruction.
Across abundant experiments, it improves Diffusion Policy and ACT in all four categories, achieving average performance comparable to their point-cloud counterparts while requiring only RGB and robot state at deployment.
Ablations support each component, highlighting the value of observation-aligned, task-relevant geometric supervision.
The current implementation uses predefined workspace bounds; automating their selection remains future work.

% \section*{AI Use Statement}

% Generative AI tools were used to assist with language polishing, improving the clarity and presentation of the manuscript, and supporting literature retrieval and discovery. All technical content, research ideas, experimental design, analysis, and conclusions were developed and verified by the authors. The authors take full responsibility for the content of the paper.

% \section*{Reproducibility Statement}

% To support reproducibility, we provide our code implementation in the supplementary materials. Detailed hyperparameters, training configurations, data collection procedures, and evaluation protocols are documented in \Cref{sec:imp}. The official codebase will be made publicly available upon acceptance.



\bibliography{main}

@inproceedings{chi2023diffusion,
  title     = {Diffusion Policy: Visuomotor Policy Learning via Action Diffusion},
  author    = {Chi, Cheng and Feng, Siyuan and Du, Yilun and Xu, Zhenjia
               and Cousineau, Eric and Burchfiel, Benjamin C. M.
               and Song, Shuran},
  booktitle = {Proceedings of Robotics: Science and Systems},
  year      = {2023},
  address   = {Daegu, Republic of Korea},
  month     = {July},
  doi       = {10.15607/RSS.2023.XIX.026}
}

@inproceedings{zhao2023act,
  title     = {Learning Fine-Grained Bimanual Manipulation with Low-Cost Hardware},
  author    = {Zhao, Tony Z. and Kumar, Vikash and Levine, Sergey and Finn, Chelsea},
  booktitle = {Proceedings of Robotics: Science and Systems},
  year      = {2023},
  address   = {Daegu, Republic of Korea},
  month     = {July},
  doi       = {10.15607/RSS.2023.XIX.016}
}

@inproceedings{ze2024dp3,
  title     = {3D Diffusion Policy: Generalizable Visuomotor Policy Learning
               via Simple 3D Representations},
  author    = {Ze, Yanjie and Zhang, Gu and Zhang, Kangning and Hu, Chenyuan
               and Wang, Muhan and Xu, Huazhe},
  booktitle = {Proceedings of Robotics: Science and Systems},
  year      = {2024},
  address   = {Delft, Netherlands},
  month     = {July},
  doi       = {10.15607/RSS.2024.XX.067}
}

@article{li2025qdepth,
  title   = {{QDepth-VLA}: Quantized Depth Prediction as Auxiliary
             Supervision for Vision-Language-Action Models},
  author  = {Li, Yixuan and Chen, Yuhui and Zhou, Mingcai and Li, Haoran
             and Zhang, Zhengtao and Zhao, Dongbin},
  journal = {arXiv preprint arXiv:2510.14836},
  year    = {2025}
}

@inproceedings{he2026foresight,
  title     = {3D Dynamics-Aware Manipulation:
               Endowing Manipulation Policies with 3D Foresight},
  author    = {He, Yuxin and Zhang, Ruihao and Wu, Xianzu and Zhang, Zhiyuan
               and Ding, Cheng and Nie, Qiang},
  booktitle = {IEEE International Conference on Robotics and Automation},
  year      = {2026}
}

@article{kim2026pri4r,
  title   = {{Pri4R}: Learning World Dynamics for Vision-Language-Action
             Models with Privileged 4D Representation},
  author  = {Kim, Jisoo and Cho, Jungbin and Chu, Sanghyeok and Bal, Ananya
             and Kim, Jinhyung and Lee, Gunhee and Lee, Sihaeng
             and Kim, Seung Hwan and Han, Bohyung and Lee, Hyunmin
             and Jeni, Laszlo A. and Kim, Seungryong},
  journal = {arXiv preprint arXiv:2603.01549},
  year    = {2026}
}

@article{li2025spatialforcing,
  title   = {Spatial Forcing: Implicit Spatial Representation Alignment for Vision-Language-Action Model},
  author  = {Li, Fuhao and Song, Wenxuan and Zhao, Han and Wang, Jingbo
             and Ding, Pengxiang and Wang, Donglin and Zeng, Long
             and Li, Haoang},
  journal = {arXiv preprint arXiv:2510.12276},
  year    = {2025}
}

@article{vuong2025evggt,
  title   = {Improving Robotic Manipulation with Efficient
             Geometry-Aware Vision Encoder},
  author  = {Vuong, An Dinh and Vu, Minh Nhat and Reid, Ian},
  journal = {arXiv preprint arXiv:2509.15880},
  year    = {2025}
}

@article{zhao2025statefree,
  title   = {Do You Need Proprioceptive States in Visuomotor Policies?},
  author  = {Zhao, Juntu and Lu, Wenbo and Zhang, Di and Liu, Yufeng
             and Liang, Yushen and Zhang, Tianluo and Cao, Yifeng
             and Xie, Junyuan and Hu, Yingdong and Wang, Shengjie
             and Guo, Junliang and Wang, Dequan and Gao, Yang},
  journal = {arXiv preprint arXiv:2509.18644},
  year    = {2025}
}

@inproceedings{seo2023mvmwm,
  title     = {Multi-View Masked World Models for Visual Robotic Manipulation},
  author    = {Seo, Younggyo and Kim, Junsu and James, Stephen and Lee, Kimin
               and Shin, Jinwoo and Abbeel, Pieter},
  booktitle = {Proceedings of the 40th International Conference on Machine Learning},
  volume    = {202},
  pages     = {30613--30632},
  publisher = {PMLR},
  year      = {2023}
}

@article{almuzairee2025mad,
  title   = {Merging and Disentangling Views in Visual Reinforcement
             Learning for Robotic Manipulation},
  author  = {Almuzairee, Abdulaziz and Patil, Rohan and Bhatt, Dwait
             and Christensen, Henrik I.},
  journal = {arXiv preprint arXiv:2505.04619},
  year    = {2025}
}

@article{wang2026track4action,
  title   = {{Track4Action}: Distilling World-Centric 3D Tracker into Vision-Language-Action Policies},
  author  = {Wang, Chenyi and Wang, Xinkai and Lin, Bokai and Tian, Jialin and Zhang, Fucheng and Lu, Cewu and Yang, Lixin},
  journal = {arXiv preprint arXiv:2608.03727},
  year    = {2026},
  url     = {https://arxiv.org/abs/2608.03727},
  note    = {arXiv:2608.03727v1, submitted 2026-08-04}
}

@article{yang2026mvucf,
  title   = {Multi-View Unified Camera Fields: Geometry-Shaped Action-Facing Representations for RGB-Only Multi-Camera VLA Policies},
  author  = {Yang, Jiarui and Lu, Yehao and Su, Yuning and Xie, Yufeng and Zhong, Yu and Lan, Haiyu and Hao, Tianjing and Lu, Kaixiang and Lin, Peiwen and Wang, Chuang and Li, Enyu and Liang, Junwei},
  journal = {arXiv preprint arXiv:2608.01826},
  year    = {2026},
  url     = {https://arxiv.org/abs/2608.01826},
  note    = {arXiv:2608.01826v1, submitted 2026-08-03}
}

@article{he2026elan4d,
  title   = {{ELAN4D}: Embodiment-Centric 4D Supervision for Vision-Language-Action Models via Plug-and-Play Adaptation},
  author  = {He, Zeyuan and Yang, Bowen and Fang, Zhirui and Zhou, Keru and Jiang, Lei and Qian, Jingjing and Mo, Fan and Yan, Junchi and Torr, Philip and Li, Xiu and Jiang, Li and Yu, Jialin},
  journal = {arXiv preprint arXiv:2605.30484},
  year    = {2026},
  url     = {https://arxiv.org/abs/2605.30484},
  note    = {arXiv:2605.30484v1}
}

@article{ding2026mindvla,
  title   = {{Mind-VLA}: Instruction-Aware Spatial Representation Alignment for Vision-Language-Action Models},
    author={Xingyu Ding and Yuzhong Zhao and Yang Wu and Chunhai Zhao and Chaoyang Zhao and Yifan Zhang and Jian Cheng},
  journal = {arXiv preprint arXiv:2608.04633},
  year    = {2026},
  url     = {https://arxiv.org/abs/2608.04633},
  note    = {arXiv:2608.04633v1, submitted 2026-08-05}
}

@article{shi2026thinkvla,
  title   = {{3DThinkVLA}: Endowing Vision-Language-Action Models with Latent 3D Priors via 3D-Thinking-Guided Co-training},
  author  = {Shi, Jiaxin and Zhang, Xidong and Zhu, Fucai and Li, Zhe and Zhu, Siyu and Yuan, Weihao},
  journal = {arXiv preprint arXiv:2606.04436},
  year    = {2026},
  url     = {https://arxiv.org/abs/2606.04436},
  note    = {arXiv:2606.04436v1}
}

@article{qian2025geopredict,
  title   = {{GeoPredict}: Leveraging Predictive Kinematics and 3D Gaussian Geometry for Precise VLA Manipulation},
  author  = {Qian, Jingjing and Han, Boyao and Shi, Chen and Xiao, Lei and Yang, Long and Shi, Shaoshuai and Jiang, Li},
  journal = {arXiv preprint arXiv:2512.16811},
  year    = {2025},
  url     = {https://arxiv.org/abs/2512.16811}
}

@article{wang2026vega,
  title   = {{VEGA}: Visual Encoder Grounding Alignment for Spatially-Aware Vision-Language-Action Models},
  author  = {Wang, Hao and Wei, Xiaobao and He, Jingyang and Bai, Chengyu and Fan, Chun-Kai and Cao, Jiajun and Chen, Jintao and Li, Ying and Rong, Shanyu and Lu, Ming and Ju, Xiaozhu and Tang, Jian and Zhang, Shanghang},
  journal = {arXiv preprint arXiv:2605.10485},
  year    = {2026},
  url     = {https://arxiv.org/abs/2605.10485}
}

@article{qian2025gp3,
  title   = {{GP3}: A 3D Geometry-Aware Policy with Multi-View Images for Robotic Manipulation},
  author  = {Qian, Quanhao and Zhao, Guoyang and Zhang, Gongjie and Wang, Jiuniu and Xu, Ran and Gao, Junlong and Zhao, Deli},
  journal = {arXiv preprint arXiv:2509.15733},
  year    = {2025},
  url     = {https://arxiv.org/abs/2509.15733}
}

@article{xu2026actiongeometry,
  title   = {Action-Geometry Prediction with 3D Geometric Prior for Bimanual Manipulation},
  author  = {Xu, Chongyang and Li, Haipeng and Cheng, Shen and Hu, Jingyu and Fan, Haoqiang and Feng, Ziliang and Liu, Shuaicheng},
  journal = {arXiv preprint arXiv:2602.23814},
  year    = {2026},
  url     = {https://arxiv.org/abs/2602.23814}
}

@article{li2025pointvla,
  title   = {{PointVLA}: Injecting the 3D World into Vision-Language-Action Models},
  author  = {Li, Chengmeng and Wen, Junjie and Peng, Yan and Peng, Yaxin and Feng, Feifei and Zhu, Yichen},
  journal = {arXiv preprint arXiv:2503.07511},
  year    = {2025},
  url     = {https://arxiv.org/abs/2503.07511}
}

@article{lu2026gap,
  title   = {When Would Vision-Proprioception Policies Fail in Robotic Manipulation?},
  author  = {Lu, Jingxian and Xia, Wenke and Wu, Yuxuan and Lu, Zhiwu and Hu, Di},
  journal = {arXiv preprint arXiv:2602.12032},
  year    = {2026},
  url     = {https://arxiv.org/abs/2602.12032},
  note    = {arXiv:2602.12032v1; accepted by ICLR 2026}
}

@article{li2026revip,
  title   = {{ReViP}: Mitigating False Completion in Vision-Language-Action Models with Vision-Proprioception Rebalance},
  author  = {Li, Zhuohao and Li, Yinghao and Jiang, Jian-Jian and Zhou, Lang and Zhang, Tianyu and Yin, Jiadong and Lin, Mu and Wei, Yi-Lin and Zheng, Wei-Shi},
  journal = {arXiv preprint arXiv:2601.16667},
  year    = {2026},
  url     = {https://arxiv.org/abs/2601.16667}
}

@article{xie2026multicamera,
  title   = {Multi-Camera View Scaling for Data-Efficient Robot Imitation Learning},
  author  = {Xie, Yichen and Wang, Yixiao and Zhao, Shuqi and Wu, Cheng-En and Tomizuka, Masayoshi and Xie, Jianwen and Fang, Hao-Shu},
  journal = {arXiv preprint arXiv:2604.00557},
  year    = {2026},
  url     = {https://arxiv.org/abs/2604.00557},
  note    = {arXiv:2604.00557v1}
}

@article{jangir2022lookcloser,
  title   = {Look Closer: Bridging Egocentric and Third-Person Views with Transformers for Robotic Manipulation},
  author  = {Jangir, Rishabh and Hansen, Nicklas and Ghosal, Sambaran and Jain, Mohit and Wang, Xiaolong},
  journal = {arXiv preprint arXiv:2201.07779},
  year    = {2022},
  url     = {https://arxiv.org/abs/2201.07779}
}

@article{zhu2024spa,
  title   = {{SPA}: 3D Spatial-Awareness Enables Effective Embodied Representation},
  author  = {Zhu, Haoyi and Yang, Honghui and Wang, Yating and Yang, Jiange and Wang, Limin and He, Tong},
  journal = {arXiv preprint arXiv:2410.08208},
  year    = {2024},
  url     = {https://arxiv.org/abs/2410.08208},
  note    = {arXiv:2410.08208v3}
}

@article{wang2024dpr,
  title   = {Visual Robotic Manipulation with Depth-Aware Pretraining},
  author  = {Wang, Wanying and Li, Jinming and Zhu, Yichen and Xu, Zhiyuan and Che, Zhengping and Peng, Yaxin and Shen, Chaomin and Liu, Dong and Feng, Feifei and Tang, Jian},
  journal = {arXiv preprint arXiv:2401.09038},
  year    = {2024},
  url     = {https://arxiv.org/abs/2401.09038}
}

@article{qian2024threedmvp,
  title   = {{3D-MVP}: 3D Multiview Pretraining for Robotic Manipulation},
  author  = {Qian, Shengyi and Mo, Kaichun and Blukis, Valts and Fouhey, David F. and Fox, Dieter and Goyal, Ankit},
  journal = {arXiv preprint arXiv:2406.18158},
  year    = {2024},
  url     = {https://arxiv.org/abs/2406.18158}
}

@inproceedings{zhang2026diffuview,
  title     = {{DiffuView}: Multi-View Diffusion Pretraining for 3D Aware Robotic Manipulation},
  author    = {Zhang, Kaizhao and Niu, Tian and Liu, Tianyu and Guo, Chenen and Xu, Zijun and Hu, Qingda and Ding, Wenchao},
  booktitle = {Proceedings of the IEEE/CVF Conference on Computer Vision and Pattern Recognition (CVPR)},
  year      = {2026},
  url       = {https://openaccess.thecvf.com/content/CVPR2026/html/Zhang_DiffuView_Multi-View_Diffusion_Pretraining_for_3D_Aware_Robotic_Manipulation_CVPR_2026_paper.html}
}

@article{pang2024depthhelps,
  title   = {Depth Helps: Improving Pre-Trained RGB-Based Policy with Depth Information Injection},
  author  = {Pang, Xincheng and Xia, Wenke and Wang, Zhigang and Zhao, Bin and Hu, Di and Wang, Dong and Li, Xuelong},
  journal = {arXiv preprint arXiv:2408.05107},
  year    = {2024},
  url     = {https://arxiv.org/abs/2408.05107}
}

@inproceedings{shridhar2023peract,
  title     = {Perceiver-Actor: A Multi-Task Transformer for Robotic Manipulation},
  author    = {Shridhar, Mohit and Manuelli, Lucas and Fox, Dieter},
  booktitle = {Proceedings of the 6th Conference on Robot Learning},
  volume    = {205},
  series    = {Proceedings of Machine Learning Research},
  pages     = {785--799},
  publisher = {PMLR},
  year      = {2023},
  url       = {https://proceedings.mlr.press/v205/shridhar23a.html}
}

@article{goyal2023rvt,
  title   = {{RVT}: Robotic View Transformer for 3D Object Manipulation},
  author  = {Goyal, Ankit and Xu, Jie and Guo, Yijie and Blukis, Valts and Chao, Yu-Wei and Fox, Dieter},
  journal = {arXiv preprint arXiv:2306.14896},
  year    = {2023},
  url     = {https://arxiv.org/abs/2306.14896}
}

@article{gervet2023act3d,
  title   = {{Act3D}: 3D Feature Field Transformers for Multi-Task Robotic Manipulation},
  author  = {Gervet, Theophile and Xian, Zhou and Gkanatsios, Nikolaos and Fragkiadaki, Katerina},
  journal = {arXiv preprint arXiv:2306.17817},
  year    = {2023},
  url     = {https://arxiv.org/abs/2306.17817}
}

@article{ke2024diffuseractor,
  title   = {3D Diffuser Actor: Policy Diffusion with 3D Scene Representations},
  author  = {Ke, Tsung-Wei and Gkanatsios, Nikolaos and Fragkiadaki, Katerina},
  journal = {arXiv preprint arXiv:2402.10885},
  year    = {2024},
  url     = {https://arxiv.org/abs/2402.10885}
}

@article{wang2024rise,
  title   = {{RISE}: 3D Perception Makes Real-World Robot Imitation Simple and Effective},
  author  = {Wang, Chenxi and Fang, Hongjie and Fang, Hao-Shu and Lu, Cewu},
  journal = {arXiv preprint arXiv:2404.12281},
  year    = {2024},
  url     = {https://arxiv.org/abs/2404.12281}
}

@article{jia2025pointmappolicy,
  title   = {{PointMapPolicy}: Structured Point Cloud Processing for Multi-Modal Imitation Learning},
  author  = {Jia, Xiaogang and Wang, Qian and Wang, Anrui and Wang, Han A. and Gyenes, Bal{\'a}zs and Gospodinov, Emiliyan and Jiang, Xinkai and Li, Ge and Zhou, Hongyi and Liao, Weiran and Huang, Xi and Beck, Maximilian and Reuss, Moritz and Lioutikov, Rudolf and Neumann, Gerhard},
  journal = {arXiv preprint arXiv:2510.20406},
  year    = {2025},
  url     = {https://arxiv.org/abs/2510.20406}
}

@misc{yang2026robo3r,
      title={Robo3R: Enhancing Robotic Manipulation with Accurate Feed-Forward 3D Reconstruction}, 
      author={Sizhe Yang and Linning Xu and Hao Li and Juncheng Mu and Jia Zeng and Dahua Lin and Jiangmiao Pang},
      year={2026},
      eprint={2602.10101},
      archivePrefix={arXiv},
      primaryClass={cs.RO},
      url={https://arxiv.org/abs/2602.10101}, 
}

@inproceedings{black2025pi0,
  title     = {{$\pi_0$}: A Vision-Language-Action Flow Model
               for General Robot Control},
  author    = {Black, Kevin and Brown, Noah and Driess, Danny
               and Esmail, Adnan and Equi, Michael and Finn, Chelsea
               and Fusai, Niccolo and Groom, Lachy and Hausman, Karol
               and Ichter, Brian and Jakubczak, Szymon and Jones, Tim
               and Ke, Liyiming and Levine, Sergey and Li-Bell, Adrian
               and Mothukuri, Mohith and Nair, Suraj and Pertsch, Karl
               and Shi, Lucy Xiaoyang and Tanner, James and Vuong, Quan
               and Walling, Anna and Wang, Haohuan and Zhilinsky, Ury},
  booktitle = {Proceedings of Robotics: Science and Systems},
  year      = {2025},
  url       = {https://arxiv.org/abs/2410.24164}
}

@inproceedings{liu2025rdt,
  title     = {{RDT-1B}: A Diffusion Foundation Model
               for Bimanual Manipulation},
  author    = {Liu, Songming and Wu, Lingxuan and Li, Bangguo
               and Tan, Hengkai and Chen, Huayu and Wang, Zhengyi
               and Xu, Ke and Su, Hang and Zhu, Jun},
  booktitle = {International Conference on Learning Representations},
  year      = {2025},
  url       = {https://openreview.net/forum?id=yAzN4tz7oI}
}

@inproceedings{ze2023gnfactor,
  title = {{GNFactor}: Multi-Task Real Robot Learning with Generalizable Neural Feature Fields},
  author = {Ze, Yanjie and Yan, Ge and Wu, Yueh-Hua and Macaluso, Annabella and Ge, Yuying and Ye, Jianglong and Hansen, Nicklas and Li, Li Erran and Wang, Xiaolong},
  booktitle = {Proceedings of The 7th Conference on Robot Learning},
  volume = {229},
  series = {Proceedings of Machine Learning Research},
  pages = {284--301},
  publisher = {PMLR},
  year = {2023},
  url = {https://proceedings.mlr.press/v229/ze23a.html}
}

@inproceedings{zhu2023groot,
  title = {Learning Generalizable Manipulation Policies with Object-Centric {3D} Representations},
  author = {Zhu, Yifeng and Jiang, Zhenyu and Stone, Peter and Zhu, Yuke},
  booktitle = {Proceedings of The 7th Conference on Robot Learning},
  volume = {229},
  series = {Proceedings of Machine Learning Research},
  pages = {3418--3433},
  publisher = {PMLR},
  year = {2023},
  url = {https://proceedings.mlr.press/v229/zhu23b.html}
}

@article{jia2024lift3d,
  title = {{Lift3D} Foundation Policy: Lifting {2D} Large-Scale Pretrained Models for Robust {3D} Robotic Manipulation},
  author = {Jia, Yueru and Liu, Jiaming and Chen, Sixiang and Gu, Chenyang and Wang, Zhilue and Luo, Longzan and Lee, Lily and Wang, Pengwei and Wang, Zhongyuan and Zhang, Renrui and Zhang, Shanghang},
  journal = {arXiv preprint arXiv:2411.18623},
  year = {2024},
  url = {https://arxiv.org/abs/2411.18623}
}

@inproceedings{wang2025monolift,
  title = {{MonoLift}: Learning {3D} Manipulation Policies from Monocular {RGB} via Distillation},
  author = {Wang, Ziru and Wang, Mengmeng and Dai, Guang and Long, Yongliu and Wang, Jingdong},
  booktitle = {Advances in Neural Information Processing Systems},
  volume = {38},
  year = {2025},
  url = {https://papers.nips.cc/paper_files/paper/2025/file/07bdb30a6aa9792c404d5d62b626f9d7-Paper-Conference.pdf}
}

@article{ze2023rl3d,
  title = {Visual Reinforcement Learning with Self-Supervised {3D} Representations},
  author = {Ze, Yanjie and Hansen, Nicklas and Chen, Yinbo and Jain, Mohit and Wang, Xiaolong},
  journal = {IEEE Robotics and Automation Letters},
  year = {2023},
  url = {https://arxiv.org/abs/2210.07241}
}

@inproceedings{zhou2023spartn,
  title = {{NeRF} in the Palm of Your Hand: Corrective Augmentation for Robotics via Novel-View Synthesis},
  author = {Zhou, Allan and Kim, Moo Jin and Wang, Lirui and Florence, Pete and Finn, Chelsea},
  booktitle = {Proceedings of the IEEE/CVF Conference on Computer Vision and Pattern Recognition},
  year = {2023},
  url = {https://openaccess.thecvf.com/content/CVPR2023/html/Zhou_NeRF_in_the_Palm_of_Your_Hand_Corrective_Augmentation_for_CVPR_2023_paper.html}
}

@article{liu2026sam3dalignment,
  title = {{SAM3D}-Guided Object-Centric Representation Alignment for Vision-Language-Action Models},
  author = {Liu, Zonghe and Jie, Shanyuan and Sun, Xiaoquan and Cao, Chen and Xu, Zetian and Liu, Zongsheng and Chen, Jiayu},
  journal = {arXiv preprint arXiv:2607.25912},
  year = {2026},
  url = {https://arxiv.org/abs/2607.25912}
}

@article{song2025reconvla,
  title = {{ReconVLA}: Reconstructive Vision-Language-Action Model as Effective Robot Perceiver},
  author = {Song, Wenxuan and Zhou, Ziyang and Zhao, Han and Chen, Jiayi and Ding, Pengxiang and Yan, Haodong and Huang, Yuxin and Tang, Feilong and Wang, Donglin and Li, Haoang},
  journal = {arXiv preprint arXiv:2508.10333},
  year = {2025},
  url = {https://arxiv.org/abs/2508.10333}
}

@inproceedings{akinola2020multicamera,
  title = {Learning Precise {3D} Manipulation from Multiple Uncalibrated Cameras},
  author = {Akinola, Iretiayo and Varley, Jacob and Kalashnikov, Dmitry},
  booktitle = {IEEE International Conference on Robotics and Automation},
  year = {2020},
  url = {https://arxiv.org/abs/2002.09107}
}

@inproceedings{hsu2022hands,
  title = {Vision-Based Manipulators Need to Also See from Their Hands},
  author = {Hsu, Kyle and Kim, Moo Jin and Rafailov, Rafael and Wu, Jiajun and Finn, Chelsea},
  booktitle = {International Conference on Learning Representations},
  year = {2022},
  url = {https://arxiv.org/abs/2203.12677}
}

@article{dunion2024multiview,
  title = {Multi-view Disentanglement for Reinforcement Learning with Multiple Cameras},
  author = {Dunion, Mhairi and Albrecht, Stefano V.},
  journal = {Reinforcement Learning Journal},
  volume = {2},
  pages = {498--515},
  year = {2024},
  url = {https://arxiv.org/abs/2404.14064}
}

@inproceedings{zhang2026cortical,
  title = {Cortical Policy: A Dual-Stream View Transformer for Robotic Manipulation},
  author = {Zhang, Xuening and Lv, Qi and Deng, Xiang and Zhang, Miao and Liu, Xingbo and Nie, Liqiang},
  booktitle = {International Conference on Learning Representations},
  year = {2026},
  url = {https://arxiv.org/abs/2603.21051}
}

@inproceedings{lin2024spawnnet,
  title = {{SpawnNet}: Learning Generalizable Visuomotor Skills from Pre-trained Networks},
  author = {Lin, Xingyu and So, John and Mahalingam, Sashwat and Liu, Fangchen and Abbeel, Pieter},
  booktitle = {IEEE International Conference on Robotics and Automation},
  year = {2024},
  url = {https://arxiv.org/abs/2307.03567}
}

@inproceedings{dehaan2019causal,
  title = {Causal Confusion in Imitation Learning},
  author = {de Haan, Pim and Jayaraman, Dinesh and Levine, Sergey},
  booktitle = {Advances in Neural Information Processing Systems},
  volume = {32},
  year = {2019},
  url = {https://papers.nips.cc/paper_files/paper/2019/hash/947018640bf36a2bb609d3557a285329-Abstract.html}
}

@article{zheng2026gift,
  title = {{GIFT}: Guided Intermediate Feature Training via Action-Oriented Structural Supervision for Robotic Manipulation},
  author = {Zheng, Yupeng and Li, Xiang and Gu, Songen and Zheng, Yuhang and Tian, Shuai and Li, Weize and Wang, Linbo and Li, Chaoyue and Zhang, Qichao and Li, Haoran and Xia, Zhongpu and Zhang, Ya-Qin and Yan, Shuicheng and Zhao, Dongbin},
  journal = {arXiv preprint arXiv:2609.04193},
  year = {2026},
  url = {https://arxiv.org/abs/2609.04193}
}

@article{gao2026focalvla,
  title = {{FOCAL-VLA}: Subtask-Guided Geometry Distillation and Implicit World Modeling for Vision-Language-Action Models},
  author = {Gao, Zhiyuan and Wen, Di and Zhan, Yanxiang and Khoshnazar, Mohammad and Sch\"afer, Jeroen and Peng, Kunyu and Beetz, Michael},
  journal = {arXiv preprint arXiv:2609.21228},
  year = {2026},
  url = {https://arxiv.org/abs/2609.21228}
}

@article{jiang2026gaussiandreampp,
  title = {{GaussianDream++}: Efficient {3D} Gaussian World Modeling for Robotic Manipulation},
  author = {Jiang, Yuqing and Zhang, Zijian and Zhou, Weitao and Wang, Jiawei and He, Junjie and Yang, Lei and Qing, Haifang and Liu, Si and Zhao, Ding and Luo, Ping and Yu, Haibao},
  journal = {arXiv preprint arXiv:2608.25659},
  year = {2026},
  url = {https://arxiv.org/abs/2608.25659}
}

@article{zhang2026gaussianwam,
  title = {{GaussianWAM}: Distilling Geometry and Semantics from {3D} Gaussian Fields into World-Action Models},
  author = {Zhang, Zijian and Jiang, Yuqing and Zhou, Weitao and Li, Minglei and Zhang, Jinhao and Mu, Yao and Li, Xiaofan and Zhao, Hao and Yu, Haibao},
  journal = {arXiv preprint arXiv:2608.24714},
  year = {2026},
  url = {https://arxiv.org/abs/2608.24714}
}
\bibliographystyle{iclr2027_conference}

\appendix
\clearpage
\section{Additional Related work}

\paragraph{Structured 3D Manipulation.}
\citet{zhu2023groot} use object-centric 3D representations to improve robustness to background and camera changes. Lift3D reconstructs depth in task-relevant image regions before adapting a pretrained 2D model to encode point clouds~\citep{jia2024lift3d}. PointVLA injects point-cloud features into a pretrained vision-language-action model~\citep{li2025pointvla}. These methods retain explicit 3D representations for policy inference, whereas LEAP uses geometry as training supervision for a policy that takes RGB images and proprioception as inputs.

\paragraph{Other Routes to Spatial Supervision.}
\citet{ze2023rl3d} pretrain voxel-based 3D representations before visual reinforcement learning, while \citet{zhou2023spartn} use NeRF-synthesized wrist views and corrective actions to augment demonstrations for an RGB policy. Depth Helps trains with depth and predicts virtual depth from RGB at inference~\citep{pang2024depthhelps}. ReconVLA instead reconstructs image regions around manipulated objects to focus visual attention~\citep{song2025reconvla}. LEAP uses full and view-matched partial point-cloud reconstruction as auxiliary supervision during policy training.

\paragraph{Camera Fusion and View-Specific Features.}
\citet{jangir2022lookcloser} combine external and wrist images through cross-view attention. Multi-view Disentanglement learns both shared and camera-specific features to support deployment with fewer views~\citep{dunion2024multiview}. Cortical Policy combines geometrically aligned static-view features with a dynamic egocentric stream~\citep{zhang2026cortical}. LEAP addresses a complementary question by matching its partial geometric target to the head--wrist views retained during masking.

\paragraph{State Conditioning and Geometric Prediction.}
\citet{dehaan2019causal} show how predictive observation cues can lead imitation policies to learn incorrect causal relationships. Robo3R uses RGB images and robot state to reconstruct metric scene geometry for downstream manipulation~\citep{yang2026robo3r}. In LEAP, reconstruction serves as supervision for the visual encoder, so the auxiliary decoder excludes robot state while the action predictor retains it.

\section{Theoretical Analysis}

\subsection{Proof of \Cref{prop:state_nonidentifiability}}
\label{app:state_nonidentifiability}

\begin{proposition}[State-induced non-identifiability]
% \label{prop:state_nonidentifiability}
Let $d:\mathcal{G}\times\mathcal{G}\to[0,\infty)$ be a measurable reconstruction cost, with prediction first and target second; no metric properties are required.
For the population risk in \Cref{eq:optimal_reconstruction_risk}, with the infimum taken over all measurable decoders,
\begin{equation}
    \mathcal{R}_{d}(G\mid Z,S)
    \leq
    \mathcal{R}_{d}(G\mid Z).
    % \label{eq:risk_monotonicity}
\end{equation}
For $\mathcal{G}=\mathbb{R}^{m}$,
$\mathbb{E}[\|G\|_2^2]<\infty$, and squared Euclidean distortion
$d_2(\widehat G,G)=\|\widehat G-G\|_2^2$,
\begin{equation}
\begin{aligned}
    \mathcal{R}_{2}(G\mid Z)
    -
    \mathcal{R}_{2}(G\mid Z,S) 
     =
    \mathbb{E}\!\left[
        \left\|
            \mathbb{E}[G\mid Z,S]
            -
            \mathbb{E}[G\mid Z]
        \right\|_2^2
    \right].
\end{aligned}
% \label{eq:state_risk_gap}
\end{equation}
\end{proposition}

\begin{proof}
Every measurable visual-only decoder $g(Z)$ can be represented
by a state-conditioned decoder $\widetilde g(Z,S)=g(Z)$.
Taking infima over the respective decoder classes yields
\Cref{eq:risk_monotonicity}.

Under squared Euclidean distortion, the finite-second-moment
assumption ensures that the optimal predictions are the conditional
means $\mu_Z=\mathbb{E}[G\mid Z]$ and
$\mu_{ZS}=\mathbb{E}[G\mid Z,S]$.
Since $\mathbb{E}[G-\mu_{ZS}\mid Z,S]=0$,
the cross term in the squared-norm expansion of
$G-\mu_Z=(G-\mu_{ZS})+(\mu_{ZS}-\mu_Z)$
has zero expectation. Therefore,
\begin{equation}
\begin{aligned}
    \mathbb{E}\!\left[\|G-\mu_Z\|_2^2\right]
    &=
    \mathbb{E}\!\left[\|G-\mu_{ZS}\|_2^2\right] \\
    &\quad+
    \mathbb{E}\!\left[\|\mu_{ZS}-\mu_Z\|_2^2\right].
\end{aligned}
\end{equation}
Substituting the corresponding optimal risks proves
\Cref{eq:state_risk_gap}.
\end{proof}

\subsection{Proof of \Cref{prop:dominant_view}}
\label{app:dominant_view}

\begin{proposition}[Non-identifiability of view-wise contributions]
% \label{prop:dominant_view}
Let $I_h$, $I_l$, and $I_r$ denote observations from the head, left-wrist, and right-wrist cameras, respectively, and let $G$ denote the target geometry.
Suppose a head-view encoder $f_h$ and decoder $g_h$ satisfy
\[
    \mathbb{E}\!\left[
        d\!\left(g_h(f_h(I_h)),G\right)
    \right]
    \leq \varepsilon.
\]
Consider full-view models with concatenated embeddings. If their hypothesis classes admit constant wrist encoders and a decoder that applies $g_h$ to the head embedding alone, then there exists an admissible pair $(F,g)$ satisfying \[  F(I_h,I_l,I_r) = [f_h(I_h),c_l,c_r],\quad \mathbb{E}\!\left[d\!\left(g(F(I_h,I_l,I_r)),G\right)\right] \leq \varepsilon, \]
\end{proposition}

\begin{proof}
Let $\pi_h(i_h,i_l,i_r)=i_h$ be the projection onto the head
observation, and define
\[
    f^\star := f_h\circ\pi_h,
    \qquad
    g^\star := g_h.
\]
By the model-class assumption, $(f^\star,g^\star)$ is an
admissible full-view encoder--decoder pair. Its reconstruction
risk satisfies
\begin{align*}
    &\mathbb{E}\!\left[
        d\!\left(g^\star(f^\star(I_h,I_l,I_r)),G\right)
    \right] \\
    &\quad =
    \mathbb{E}\!\left[
        d\!\left(g_h(f_h(\pi_h(I_h,I_l,I_r))),G\right)
    \right] \\
    &\quad =
    \mathbb{E}\!\left[
        d\!\left(g_h(f_h(I_h)),G\right)
    \right]
    \leq \varepsilon.
\end{align*}
Moreover, for any fixed $i_h$ and any wrist observations
$(i_l,i_r)$ and $(i_l',i_r')$,
\[
    f^\star(i_h,i_l,i_r)
    =
    f_h(i_h)
    =
    f^\star(i_h,i_l',i_r').
\]
Thus, $f^\star(I_h,I_l,I_r)$ is measurable with respect to
$\sigma(I_h)$ and invariant to both wrist inputs, establishing
the claimed solution.
\end{proof}

\section{Implementation Details}\label{sec:imp}

\subsection{Task Category of RoboTwin}

We divide the RoboTwin 50 tasks in the following manner:
\begin{itemize}
    \item \textbf{Object-centric Spatial Manipulation (14):}
    \texttt{adjust\_bottle},
    \texttt{move\_pillbottle\_pad},
    \texttt{move\_stapler\_pad},
    \texttt{place\_bread\_basket},
    \texttt{place\_bread\_skillet},
    \texttt{place\_container\_plate},
    \texttt{place\_empty\_cup},
    \texttt{place\_fan},
    \texttt{place\_mouse\_pad},
    \texttt{place\_object\_scale},
    \texttt{place\_object\_stand},
    \texttt{place\_phone\_stand},
    \texttt{place\_shoe},
    \texttt{rotate\_qrcode}.

    \item \textbf{Multi-object Spatial Reasoning (14):}
    \texttt{blocks\_ranking\_rgb},
    \texttt{blocks\_ranking\_size},
    \texttt{move\_can\_pot},
    \texttt{place\_a2b\_left},
    \texttt{place\_a2b\_right},
    \texttt{place\_burger\_fries},
    \texttt{place\_can\_basket},
    \texttt{place\_cans\_plasticbox},
    \texttt{place\_dual\_shoes},
    \texttt{place\_object\_basket},
    \texttt{stack\_blocks\_three},
    \texttt{stack\_blocks\_two},
    \texttt{stack\_bowls\_three},
    \texttt{stack\_bowls\_two}.

    \item \textbf{Interaction / Articulation (14):}
    \texttt{beat\_block\_hammer},
    \texttt{click\_alarmclock},
    \texttt{click\_bell},
    \texttt{dump\_bin\_bigbin},
    \texttt{handover\_block},
    \texttt{handover\_mic},
    \texttt{hanging\_mug},
    \texttt{open\_laptop},
    \texttt{open\_microwave},
    \texttt{press\_stapler},
    \texttt{put\_object\_cabinet},
    \texttt{scan\_object},
    \texttt{stamp\_seal},
    \texttt{turn\_switch}.

    \item \textbf{Less-geometry-dependent (8):}
    \texttt{grab\_roller},
    \texttt{lift\_pot},
    \texttt{move\_playingcard\_away},
    \texttt{pick\_diverse\_bottles},
    \texttt{pick\_dual\_bottles},
    \texttt{put\_bottles\_dustbin},
    \texttt{shake\_bottle},
    \texttt{shake\_bottle\_horizontally}.
\end{itemize}

\subsection{Experimental Details}\label{subsec:task}

\begin{figure}[h]
    \centering
    \includegraphics[width=0.99\linewidth]{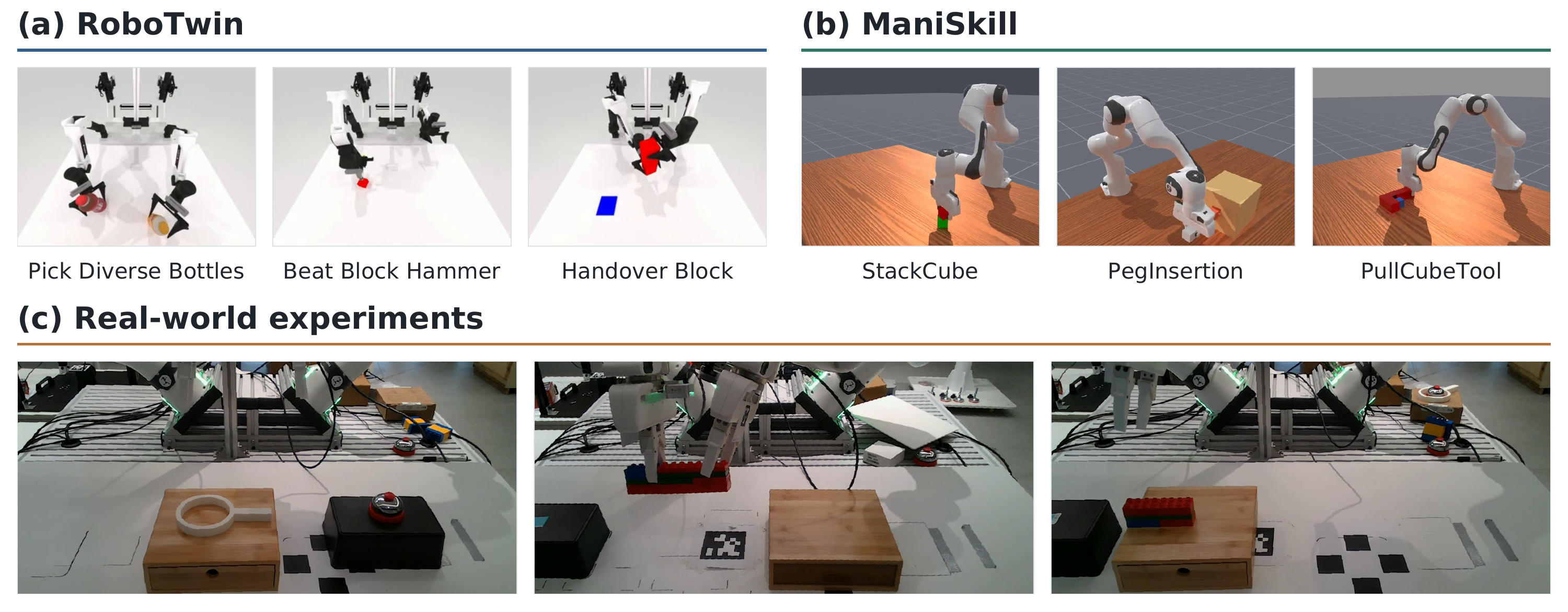}
    \caption{Evaluation environments. (a) Representative bimanual manipulation tasks from RoboTwin. (b) Representative tasks in ManiSkill. (c) Three real-world manipulation setups, including Pick Place, Assemble and Handover Blocks. Simulation panels show representative frames from benchmark demonstrations.}
    \label{fig:task}
    \vspace{-2ex}
\end{figure}
\paragraph{Task Overview.}
Figure~\ref{fig:task} illustrates the three evaluation settings. RoboTwin provides diverse bimanual manipulation tasks, while ManiSkill covers single-arm tasks including stacking, insertion, and tool use. Our real-world evaluation comprises three tasks—Pick Place, Assemble, and Handover Blocks—that test object placement, precise assembly, and coordination between arms, respectively.

\subsection{Hyperparameters}\label{sec:aes}

% Requires \usepackage{booktabs}
\begin{table*}[h]
\centering
\caption{Main hyperparameters for DP-\methodname{} and ACT-\methodname{}. Each RoboTwin task provides 50 demonstrations and each Maniskill task provides 50 demonstrations. Point clouds are used only for training supervision; both policies use RGB images and robot state at inference.}
\label{tab:robotwin-hyperparameters}
\footnotesize
\setlength{\tabcolsep}{5pt}
\renewcommand{\arraystretch}{1.12}
\begin{tabular}{@{}p{0.24\textwidth}p{0.34\textwidth}p{0.34\textwidth}@{}}
\toprule
\textbf{Hyperparameter} & \textbf{DP-\methodname{}} & \textbf{ACT-\methodname{}} \\
\midrule
Policy inputs &
\multicolumn{2}{p{0.68\textwidth}@{}}{Head, left-wrist, and right-wrist RGB; 14-D robot state} \\
Visual encoder &
\multicolumn{2}{p{0.68\textwidth}@{}}{Independent ResNet-18 per camera (GroupNorm; no pretraining)} \\
View fusion &
\multicolumn{2}{p{0.68\textwidth}@{}}{2-layer, 8-head transformer (width 512)} \\
Geometry supervision &
\multicolumn{2}{p{0.68\textwidth}@{}}{Full and partial 1024-point world-frame targets; visual-only decoder (width 512)} \\
View sampling &
\multicolumn{2}{p{0.68\textwidth}@{}}{Head retained; one wrist sampled for partial supervision} \\
\midrule
Geometry loss &
$0.5\mathcal{L}_{\mathrm{full}}+0.5\mathcal{L}_{\mathrm{partial}}$ &  $1.5\mathcal{L}_{\mathrm{full}}+1.5\mathcal{L}_{\mathrm{partial}}$\\

Policy backbone &
Conditional 1-D U-Net (widths 256/512/1024) &
CVAE Transformer (hidden width 512; latent width 32; 4 encoder and 7 decoder layers) \\
Action prediction &
8-step horizon; 3 observations; execute 6 steps &
50-step chunk from 1 observation \\
Policy objective &
Diffusion loss (DDPM; 100 training and inference steps) &
$\mathcal{L}_{1}+10\mathcal{L}_{\mathrm{KL}}$ \\
Optimizer (LR; weight decay) &
AdamW ($10^{-4}$; $10^{-6}$) &
AdamW ($10^{-5}$; $10^{-4}$) \\
Batch size & 32 & 8 \\
Training epochs & 600 & 4000 \\
\bottomrule
\end{tabular}
\end{table*}

\paragraph{Hyperparameters.}
We train DP-\methodname{} and ACT-\methodname{} on 50 demonstrations per RoboTwin task. Both policies take RGB images from a head camera and two wrist cameras, together with a 14-D robot state. Each camera has an independently trained ResNet-18 encoder, and a two-layer transformer fuses the visual features. During training, we retain the head view and sample one wrist view for partial geometry supervision. A visual-only decoder reconstructs full and partial 1024-point clouds in the world frame; the geometry loss uses a 10-epoch linear warmup. DP-\methodname{} uses a conditional 1-D U-Net and is trained for 600 epochs, while ACT-\methodname{} uses a CVAE Transformer and is trained for 4000 epochs. Point clouds and the geometry decoder are used only during training.

\paragraph{Pre-Fusion View Masking.}
Each camera is processed by an independent ResNet-18 encoder, producing camera-specific features $Z_c=f_{\theta_c}(I_c)$ without cross-view interaction.
For partial reconstruction, we apply view masking to these features \emph{before} the fusion layer, retaining the head view and one randomly selected wrist.
The fusion layer then aggregates visual information only from the retained views.
Because masking precedes all cross-view interaction, the dropped camera's features cannot reach the partial reconstruction decoder, either directly or indirectly through the retained features.
This ensures that the partial prediction is conditioned only on the selected head--wrist observations, rather than on features that have already incorporated information from all cameras.

\paragraph{Data Collection and Evaluation Protocol.}
We collect 50 expert trajectories per RoboTwin task, 100 expert trajectories per ManiSkill task, and 50 expert trajectories per real-world task.
For each simulation task, we evaluate each policy using three independent random seeds, with 100 rollouts per seed, yielding 300 evaluation rollouts per policy and task.
We compute the success rate separately for each seed and report the mean and standard deviation across the three seeds.
For real-world experiments, we use a single random seed and evaluate each policy over 20 rollouts per task, reporting the percentage of successful trials.

\subsection{One-Arm Camera Configuration}\label{subsec:one-arm}
For one-arm tasks, we use three RGB views: a head camera $h$, a wrist camera $w$, and an additional outside camera $o$ with broad scene coverage. Because either $h$ or $o$ remains available when the other is dropped, we uniformly mask any one of the three cameras:
\begin{equation}
    M\in\bigl\{\{h,w\},\{o,w\},\{h,o\}\bigr\}.
\end{equation}
When the wrist is retained, we construct the partial target from wrist-visible points, as above. When the wrist is masked, we instead use points approximately visible from the retained outside camera. Thus, the partial target is always aligned with an available view, while the two configurations containing the wrist encourage the policy to use its local geometric observations.

\section{Additional Experimental Results}

\subsection{Full Results of Ablation Study}\label{subsec:full-a}

\begin{table}[t]
    \centering
    \caption{Full ablation results on five manipulation tasks, reported as success rates (\%).}
    \label{tab:ablation-full}
    \footnotesize
    \setlength{\tabcolsep}{3.5pt}
    \renewcommand{\arraystretch}{1.15}
    \begin{tabular}{@{}lcccccc@{}}
        \toprule
        Variant &
        \shortstack{BeatBlock\\Hammer} &
        \shortstack{Handover\\Blocks} &
        \shortstack{PickDiverse\\Bottles} &
        \shortstack{PlaceCans\\Plasticbox} &
        \shortstack{PlaceObject\\Basket} &
        Average \\
        \midrule
        DP-\methodname{} (full)
            & \textbf{85} & \textbf{89} & \textbf{48}
            & \textbf{87} & \textbf{74} & \textbf{76.6} \\
        \quad w/o wrist dropout
            & 73 & 59 & 34 & 49 & 61 & 55.2 \\
        \quad w/ state in decoder
            & 68 & 70 & 36 & 64 & 53 & 58.2 \\
        \quad w/ depth reconstruction
            & 79 & 83 & 47 & 79 & 67 & 71.0 \\
        \quad w/o target alignment
            & 78 & 63 & 39 & 60 & 70 & 62.0 \\
        \midrule
        ACT-\methodname{} (full)
            & 55 & \textbf{79} & \textbf{49}
            & 20 & \textbf{42} & \textbf{49.0} \\
        \quad w/o wrist dropout
            & 49 & 63 & 43 & 13 & 25 & 38.6 \\
        \quad w/ state in decoder
            & 47 & 58 & 28 & 15 & 36 & 36.8 \\
        \quad w/ depth reconstruction
            & \textbf{59} & 70 & 43 & \textbf{24}
            & 31 & 45.4 \\
        \quad w/o target alignment
            & 51 & 72 & 47 & 17 & 35 & 44.4 \\
        \bottomrule
    \end{tabular}
\end{table}

\subsection{Sensitivity Analysis}\label{subsec:sen}

\begin{table}[h]
    \centering
    \caption{Sensitivity to the geometry reconstruction weight $\lambda$.
    Entries are success rates; Average is the arithmetic mean over five tasks.
    The best result within each policy group is shown in bold. Default $\lambda_{\mathrm{full}}$ and $\lambda_{\mathrm{partial}}$ for DP-\methodname{} is 0.5 and for ACT-\methodname{} is 1.5.  }
    \label{tab:geometry_weight_sensitivity}
    % \scriptsize
    \setlength{\tabcolsep}{5pt}
    \begin{tabular}{@{}lcccccc@{}}
        \toprule
        Task & \multicolumn{3}{c}{DP-\methodname{}}
        & \multicolumn{3}{c}{ACT-\methodname{}} \\
        \cmidrule(lr){2-4}\cmidrule(l){5-7}
        $\lambda_{\mathrm{full}}/\lambda_{\mathrm{partial}}$
        & $0.4$ & $\mathbf{0.5}$ (Default) & $0.6$
        & $1.2$ & $\mathbf{1.5}$ (Default) & $1.8$ \\
        \midrule
        BeatBlockHammer       & 0.79 & \textbf{0.85} & 0.82 & 0.51 & \textbf{0.55} & 0.52 \\
        HandoverBlock         & 0.83 & \textbf{0.89} & 0.79 & 0.72 & \textbf{0.79} & 0.69 \\
        PickDiverseBottles    & 0.45 & \textbf{0.48} & 0.41 & 0.42 & 0.49 & \textbf{0.53} \\
        PlaceCansPlasticbox   & 0.89 & 0.87 & \textbf{0.93} & 0.17 & \textbf{0.20} & 0.13 \\
        PlaceObjectBasket     & \textbf{0.77} & 0.74 & 0.72 & 0.38 & \textbf{0.42} & 0.41 \\
        \midrule
        Average               & 0.746 & \textbf{0.766} & 0.734
                              & 0.440 & \textbf{0.490} & 0.456 \\
        \bottomrule
    \end{tabular}
\end{table}

\paragraph{Sensitivity to Geometry Reconstruction Weight.}
We vary the geometry reconstruction weights $\lambda_{\mathrm{full}}$ and $\lambda_{\mathrm{partial}}$ on five RoboTwin tasks. DP-\methodname{} achieves its highest average success rate at $\lambda_{\mathrm{full}}=\lambda_{\mathrm{partial}}$=0.5 (76.6\%), compared with 74.6\% and 73.4\% at $\lambda_{\mathrm{full}}=\lambda_{\mathrm{partial}}=0.4$ and $0.6$, respectively. ACT-\methodname{} performs best on average at $\lambda_{\mathrm{full}}=\lambda_{\mathrm{partial}}=1.5$ (49.0\%), compared with 44.0\% at $\lambda_{\mathrm{full}}=\lambda_{\mathrm{partial}}=1.2$ and 45.6\% at $\lambda_{\mathrm{full}}=\lambda_{\mathrm{partial}}=1.8$. The best weight varies across individual tasks, but the average results support $\lambda_{\mathrm{full}}=\lambda_{\mathrm{partial}}=0.5$ for DP and $\lambda_{\mathrm{full}}=\lambda_{\mathrm{partial}}=1.5$ for ACT.

\subsection{Full Results on RoboTwin}

\begin{table*}[h]
\centering
\caption{Per-task success rates (\%) on the 50 RoboTwin 2.0 tasks under clean evaluation. DP, RDT, and $\pi_0$ are from the official single-task leaderboard; the other results are from our evaluations.}
\label{tab:robotwin_50_per_task_1}
\scriptsize
\setlength{\tabcolsep}{2.2pt}
\renewcommand{\arraystretch}{1.12}
\resizebox{\textwidth}{!}{%
\begin{tabular}{@{}l*{12}{r}@{}}
\toprule
Task & DP3 & DP & \shortstack{DP-\\MultiView} & DP-eVGGT & RDT & DP-SF & $\pi_0$ & GAP & DP-LEAP & \shortstack{ACT-\\pointcloud} & ACT & ACT-LEAP \\
\midrule
adjust\_bottle & 99 & 97 & 95 & 96 & 81 & 99 & 90 & 98 & 99 & 93 & 90 & 92 \\
beat\_block\_hammer & 72 & 42 & 27 & 43 & 77 & 60 & 43 & 56.3 & 85 & 58 & 43 & 55 \\
blocks\_ranking\_rgb & 4 & 0 & 0 & 1 & 3 & 4 & 19 & 3 & 6 & 0 & 1 & 1 \\
blocks\_ranking\_size & 1 & 1 & 0 & 1 & 0 & 2 & 7 & 1 & 2 & 1 & 0 & 1 \\
click\_alarmclock & 77 & 61 & 52 & 47 & 61 & 58 & 63 & 64 & 92 & 90 & 35 & 87 \\
click\_bell & 97 & 54 & 50 & 56 & 80 & 73 & 44 & 80 & 97 & 88 & 63 & 84 \\
dump\_bin\_bigbin & 87 & 49 & 43 & 51 & 64 & 59 & 83 & 61 & 89 & 93 & 61 & 89 \\
grab\_roller & 95 & 98 & 80 & 87 & 74 & 89 & 96 & 84 & 95 & 97 & 83 & 96 \\
handover\_block & 87 & 10 & 21 & 34 & 45 & 47 & 45 & 68.7 & 89 & 89 & 47 & 79 \\
handover\_mic & 99 & 53 & 46 & 59 & 90 & 69 & 98 & 71 & 100 & 99 & 88 & 99 \\
hanging\_mug & 26 & 8 & 3 & 7 & 23 & 15 & 11 & 23 & 21 & 13 & 8 & 17 \\
lift\_pot & 96 & 39 & 49 & 40 & 72 & 51 & 84 & 63 & 55 & 95 & 85 & 93 \\
move\_can\_pot & 72 & 39 & 27 & 33 & 25 & 53 & 58 & 55.3 & 67 & 70 & 30 & 65 \\
move\_pillbottle\_pad & 42 & 1 & 0 & 12 & 8 & 17 & 21 & 9 & 23 & 37 & 2 & 41 \\
move\_playingcard\_away & 69 & 47 & 47 & 38 & 43 & 50 & 53 & 41 & 60 & 74 & 37 & 66 \\
move\_stapler\_pad & 8 & 1 & 0 & 1 & 2 & 3 & 0 & 2 & 4 & 1 & 0 & 1 \\
open\_laptop & 75 & 49 & 55 & 50 & 59 & 62 & 85 & 79 & 64 & 85 & 61 & 79 \\
open\_microwave & 88 & 5 & 21 & 43 & 37 & 69 & 80 & 62 & 67 & 66 & 81 & 86 \\
pick\_diverse\_bottles & 46 & 6 & 12 & 22 & 2 & 37 & 27 & 24.3 & 48 & 56 & 3 & 49 \\
pick\_dual\_bottles & 81 & 24 & 25 & 37 & 42 & 56 & 57 & 69 & 74 & 67 & 31 & 63 \\
place\_a2b\_left & 39 & 2 & 0 & 15 & 3 & 22 & 31 & 26 & 34 & 21 & 1 & 33 \\
place\_a2b\_right & 37 & 13 & 10 & 14 & 1 & 16 & 27 & 24.7 & 21 & 16 & 0 & 23 \\
place\_bread\_basket & 23 & 14 & 6 & 9 & 10 & 27 & 17 & 20.7 & 36 & 20 & 4 & 32 \\
place\_bread\_skillet & 37 & 11 & 6 & 18 & 5 & 29 & 23 & 34 & 42 & 26 & 10 & 30 \\
place\_burger\_fries & 70 & 72 & 51 & 50 & 50 & 56 & 80 & 72 & 70 & 68 & 46 & 70 \\
place\_can\_basket & 69 & 18 & 23 & 32 & 19 & 40 & 41 & 27 & 59 & 15 & 10 & 30 \\
place\_cans\_plasticbox & 48 & 40 & 31 & 37 & 6 & 55 & 34 & 61.3 & 87 & 35 & 8 & 20 \\
place\_container\_plate & 86 & 41 & 41 & 53 & 78 & 71 & 88 & 83 & 74 & 86 & 75 & 88 \\
place\_dual\_shoes & 13 & 8 & 3 & 5 & 4 & 10 & 15 & 27 & 29 & 6 & 2 & 9 \\
place\_empty\_cup & 65 & 37 & 31 & 43 & 56 & 67 & 37 & 76.3 & 95 & 59 & 45 & 67 \\
place\_fan & 36 & 3 & 3 & 8 & 12 & 12 & 20 & 26 & 27 & 19 & 3 & 27 \\
place\_mouse\_pad & 4 & 0 & 0 & 2 & 1 & 3 & 7 & 5 & 9 & 0 & 0 & 1 \\
place\_object\_basket & 65 & 15 & 15 & 23 & 33 & 49 & 16 & 60.3 & 74 & 37 & 17 & 42 \\
place\_object\_scale & 16 & 1 & 0 & 4 & 1 & 7 & 10 & 8 & 10 & 11 & 0 & 17 \\
place\_object\_stand & 54 & 22 & 21 & 29 & 15 & 38 & 36 & 42.7 & 43 & 29 & 2 & 25 \\
place\_phone\_stand & 44 & 13 & 13 & 20 & 15 & 40 & 35 & 46 & 49 & 35 & 0 & 38 \\
place\_shoe & 49 & 23 & 23 & 31 & 35 & 37 & 28 & 49 & 45 & 19 & 2 & 28 \\
press\_stapler & 66 & 6 & 13 & 44 & 41 & 75 & 62 & 78 & 93 & 57 & 32 & 65 \\
put\_bottles\_dustbin & 60 & 22 & 22 & 41 & 21 & 51 & 54 & 61 & 63 & 44 & 35 & 44 \\
put\_object\_cabinet & 72 & 42 & 42 & 37 & 33 & 44 & 68 & 32 & 63 & 23 & 10 & 37 \\
rotate\_qrcode & 71 & 13 & 13 & 25 & 50 & 49 & 68 & 59 & 62 & 25 & 0 & 28 \\
scan\_object & 27 & 9 & 18 & 23 & 4 & 27 & 18 & 30.7 & 37 & 10 & 4 & 7 \\
shake\_bottle & 98 & 65 & 65 & 73 & 74 & 83 & 97 & 93.3 & 94 & 100 & 56 & 97 \\
shake\_bottle\_horizontally & 98 & 59 & 59 & 72 & 84 & 69 & 99 & 67 & 98 & 100 & 73 & 96 \\
stack\_blocks\_three & 1 & 0 & 0 & 1 & 2 & 3 & 17 & 2 & 4 & 1 & 0 & 1 \\
stack\_blocks\_two & 22 & 7 & 3 & 19 & 21 & 22 & 42 & 32 & 47 & 7 & 19 & 25 \\
stack\_bowls\_three & 57 & 63 & 49 & 50 & 51 & 51 & 66 & 57 & 63 & 49 & 50 & 57 \\
stack\_bowls\_two & 85 & 61 & 61 & 76 & 76 & 72 & 91 & 95.3 & 96 & 84 & 79 & 88 \\
stamp\_seal & 18 & 2 & 2 & 5 & 1 & 9 & 3 & 7 & 10 & 27 & 4 & 19 \\
turn\_switch & 46 & 36 & 30 & 35 & 35 & 33 & 27 & 51 & 56 & 23 & 8 & 20 \\
\midrule
\textbf{Average} & 55.94 & 28.04 & 26.14 & 33.04 & 34.50 & 42.80 & 46.42 & 47.38 & 56.54 & 46.48 & 28.88 & 48.14 \\
\bottomrule
\end{tabular}%
}
\end{table*}

\end{document}